\documentclass{article} 
\usepackage{iclr2027_conference,times}

\usepackage{amsmath,amsfonts,bm}

\def\eqref#1{equation~\ref{#1}}

\def\1{\bm{1}}

\DeclareMathAlphabet{\mathsfit}{\encodingdefault}{\sfdefault}{m}{sl}
\SetMathAlphabet{\mathsfit}{bold}{\encodingdefault}{\sfdefault}{bx}{n}

\usepackage{hyperref}
\usepackage{url}
\usepackage{xcolor}
\usepackage{graphicx}
\usepackage{booktabs}
\usepackage{pifont}
\usepackage[table]{xcolor}
\usepackage{booktabs}
\usepackage{multirow}
\usepackage{makecell}
\usepackage{array}

\usepackage{amsthm}
\newtheorem{proposition}{Proposition}

\renewcommand{\theproposition}{A.\arabic{proposition}}

\title{Coupling Perception and Reasoning in Federated Multimodal Graph Foundation Models}

\author{
\normalfont
Zekai Chen$^{1}$ \quad
Xun Wu$^{1}$ \quad
Hailin Zhang$^{2}$ \quad
Xunkai Li$^{1}$ \quad
Yu Liu$^{3}$ \\
Kairui Yang$^{1}$ \quad
Muyan Huang$^{4}$ \quad
Xuaner Chen$^{2}$ \quad
Rong-Hua Li$^{1}$ \quad
Guoren Wang$^{1}$ \\[1ex]
$^{1}$Beijing Institute of Technology \\
$^{2}$Shandong University \\
$^{3}$Institute of Information Engineering \\
$^{4}$University of Science and Technology Beijing \\
Zekai Chen: \texttt{zackchen02@163.com}
}

\iclrfinalcopy
\begin{document}

\maketitle

\begin{abstract}
Federated multimodal graph foundation models (GFMs) aim to adapt pretrained multimodal models to decentralized graph data, where each client owns a private multimodal graph and cannot share raw information. These models typically combine a multimodal Encoder that extracts semantic evidence from heterogeneous modalities and a graph neural network (GNN) that performs relational reasoning over graph structures. However, existing federated GFM adaptation methods mainly update graph-side modules while keeping the multimodal Encoder frozen, limiting adaptation to \emph{how information is propagated} while fixing \emph{what information is extracted}. Through empirical studies, we reveal that Encoder and GNN adaptations are not
independent: Encoder adaptation is affected by graph relations, while cross-client
module swapping reveals substantial pairing sensitivity between separately
parameterized Encoder and GNN updates. Motivated by this observation, we propose \textbf{FedCORE}, a federated adaptation framework that represents Encoder and GNN updates through a shared low-dimensional latent state. FedCORE jointly optimizes this core from multimodal and structural signals and performs federated evolution directly in the shared state space, preserving compatibility between perception and reasoning adaptations. Extensive experiments demonstrate that FedCORE reduces the Encoder--GNN pairing gap from $30.6$ to $5.9$, corresponding to an $80.7\%$ reduction over independent joint adaptation.
\end{abstract}

\section{Introduction}
\label{sec:introduction}

Multimodal graph foundation models (GFMs) have emerged as a paradigm for
jointly modeling multimodal semantics and graph relations. Combining pretrained
multimodal encoders with graph neural networks (GNNs), multimodal GFMs enable
unified perception and relational reasoning across diverse graph-based
applications~\citep{he2025unigraph2,liu2026toward}, including e-commerce
recommendation~\citep{liu2021pre}, social and content platforms
~\citep{wei2019mmgcn}. However, in practice, these graphs are distributed
across organizations or platforms and cannot be centralized because of privacy,
ownership, and communication constraints. Federated learning therefore provides
a practical paradigm for collaboratively adapting multimodal GFMs without
sharing raw data.

A natural solution jointly adapts the Encoder and GNN on each client, allowing
multimodal perception and graph reasoning to improve jointly on local tasks.
During training, each client updates both modules using its multimodal graph and
downstream objective, while federated optimization typically aggregates their
updates independently. Such module-wise aggregation fails to preserve the
relationship between Encoder and GNN updates established during training,
weakening coordination between multimodal perception and graph reasoning after
federation. Thus, the key challenge is not only enabling Encoder--GNN joint
adaptation but also preserving their coordination during federated optimization
and aggregation.

To understand this coordination, our empirical studies in
Sec.~\ref{sec:empirical} examine three questions:
\textbf{\ding{182} Encoder Adaptation.}
We first test whether stronger graph reasoning can compensate for a fixed
multimodal Encoder. Figure~\ref{fig:obs1} shows that increasing GNN capacity
cannot consistently compensate for the absence of Encoder adaptation,
demonstrating that graph reasoning alone cannot recover the missing
task-relevant multimodal information and that perception adaptation is
necessary.
\textbf{\ding{183} Graph Dependency.}
Given its necessity, we further study whether Encoder adaptation can be
optimized independently from graph structures. Figure~\ref{fig:obs2} shows
that controlled changes in graph relations systematically alter Encoder
updates even when multimodal contents and labels remain unchanged, revealing
that Encoder adaptation depends on the relational environment where
representations are utilized.
\textbf{\ding{184} Update Coupling.}
Since Encoder adaptation is affected by graph relations, we finally examine
whether locally co-adapted Encoder and GNN updates are freely interchangeable
across clients. Figure~\ref{fig:obs3} shows that locally matched Encoder--GNN
pairs outperform cross-client mismatched combinations on average, revealing
substantial pairing sensitivity between the two modules. These observations
motivate explicitly coupling Encoder and GNN adaptations through a shared
state during local optimization, rather than parameterizing their updates
separately.

Based on these observations, we propose FedCORE, a federated multimodal graph
learning framework that represents coupled multimodal perception and graph
reasoning adaptation through a shared latent state. Rather than independently
maintaining Encoder and GNN adapters, FedCORE learns a compact core jointly
optimized on each client and decoded into module-specific low-rank updates for
multimodal Encoders, fusion layers, and GNNs. The shared core preserves the
relationship between extracted multimodal evidence and its propagation through
graph structures. During federation, only the shared core is communicated and
aggregated, maintaining learned perception--reasoning coordination across
clients while retaining communication efficiency.

\textbf{Our Contributions.}
(1) \textbf{\textit{\underline{New Insight.}}}
We reveal that Encoder and GNN adaptations in federated multimodal GFMs form a coupled optimization process rather than independent modules, motivating coordinated perception--reasoning adaptation.
(2) \textbf{\textit{\underline{New Framework.}}}
We propose FedCORE, which represents Encoder and GNN adaptations through a shared low-dimensional latent state, enabling coordinated optimization and aggregation-consistent federation with low communication overhead.
(3) \textbf{\textit{\underline{SOTA Performance.}}}
Experiments across multimodal graph benchmarks and downstream tasks show that FedCORE consistently outperforms strong
federated baselines, achieving an average improvement of \textbf{2.27\%} over the strongest non-FedCORE baseline.

\section{Related Work}
\label{sec:related_work}

\subsection{Multimodal Graph Foundation Models}

Multimodal graph foundation models combine pretrained multimodal encoders with
graph neural networks to jointly capture semantic information and relational
structures \citep{he2025unigraph2,fang2025graphgpt}. Recent studies explore
cross-modal interaction, topology-aware representation learning, and
structural grounding in multimodal graphs
\citep{ning2025graph4mm,liu2026toward}. These methods validate combining
multimodal perception with graph reasoning but primarily consider centralized
training, where components are jointly optimized with complete data. We
instead study adapting and coordinating such multimodal graph models under
federated settings.

\subsection{Federated Adaptation of Graph Foundation Models}

Federated graph foundation models transfer pretrained graph learning
capabilities across decentralized clients while preserving data locality.
Recent studies explore parameter-efficient adaptation and shared knowledge
modeling for improved cross-client transferability
\citep{zhu2026towards,wu2025fedbook,zhu2026fedgala,chen2026fedlab}. Meanwhile,
federated multimodal graph learning considers challenges from heterogeneous
modalities and decentralized graph structures, with recent methods developing
adaptive strategies for multimodal federated environments
\citep{li2026mm,zhang2026towards,chen2026stage,chen2026prism,zhu2026fedtcr,li2026fedgamma}.
However, existing approaches generally treat multimodal representation learning
and graph reasoning as independent adaptation components, overlooking that
updates can become coupled during local optimization and lose coordination
after federated aggregation. FedCORE addresses this limitation by modeling
Encoder and GNN adaptations through a shared latent state that preserves their
perception--reasoning coordination throughout federated learning.

\section{Empirical Study}
\label{sec:empirical}

We first diagnose why independently adapting and aggregating the
Encoder and GNN may be insufficient in federated multimodal GFMs.
 Our studies answer three questions: whether Encoder adaptation is necessary beyond graph-side optimization, whether Encoder adaptation depends on graph relations, and whether Encoder--GNN updates should be coordinated during federation. Detailed protocols and metric definitions are provided in Appendix~\ref{app:empirical_setup}.

\subsection{Can Stronger GNNs Replace Encoder Adaptation?}

\begin{figure}[htbp]
    \centering
    \includegraphics[width=0.9\linewidth]
    {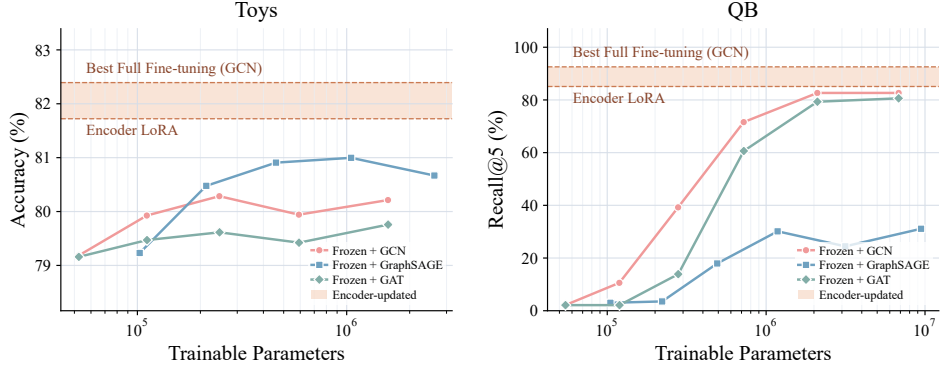}
    \vspace{-3pt}
    \setlength{\abovecaptionskip}{3pt}
    \setlength{\belowcaptionskip}{0pt}
    \caption{
    Observation 1: Stronger GNNs cannot replace Encoder adaptation.
    Performance of graph-only adaptation with increasing GNN capacity compared with Encoder-adapted models.
    }
    \label{fig:obs1}
\end{figure}
\textbf{Observation 1:}
\textit{Graph-side adaptation improves relational reasoning but cannot replace task-specific Encoder adaptation.}
We first examine whether stronger graph reasoning can compensate for a frozen multimodal Encoder. In Fig.~\ref{fig:obs1}, we gradually increase the trainable capacity of GCN, GraphSAGE, and GAT while keeping the Encoder fixed, and compare them with Encoder LoRA adaptation and full fine-tuning. Although larger GNN capacity consistently improves graph-only adaptation, it remains inferior to Encoder-adapted models. This indicates that graph reasoning cannot recover multimodal information that is not sufficiently captured by the frozen Encoder.

\subsection{Does Encoder Adaptation Depend on Relational Context?}

\begin{figure}[htbp]
    \centering
    \includegraphics[width=0.9\linewidth]
    {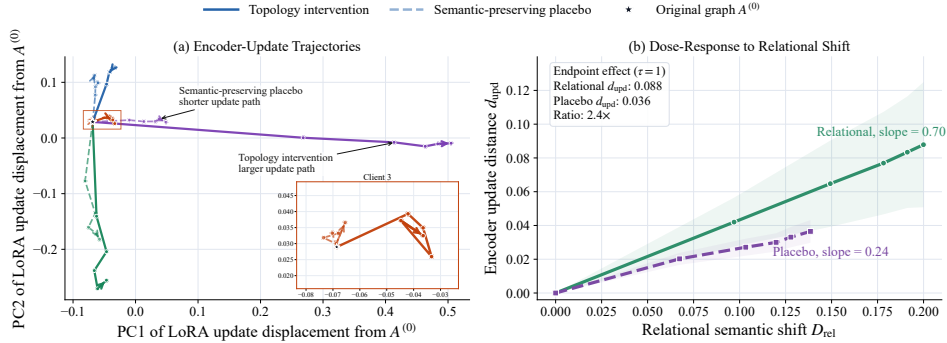}
    \caption{
    Observation 2: Relational context conditions multimodal representation adaptation.
    Encoder-LoRA update trajectories under graph interventions and semantic-preserving perturbations.
    }
    \label{fig:obs2}
\end{figure}

\textbf{Observation 2:} 
\textit{Multimodal representation adaptation is relationally conditioned rather than determined solely by modality contents.}
Given the necessity of Encoder adaptation, we further investigate whether it can be optimized independently from graph structures. We keep node features, labels, and splits unchanged while intervening only on graph topology. As shown in Fig.~\ref{fig:obs2}, stronger relational interventions induce increasingly different Encoder-LoRA update trajectories, whereas semantic-preserving perturbations introduce much smaller changes. The mean update distance reaches $0.088$ under the strongest intervention compared with $0.036$ under the low-semantic-shift control, demonstrating that Encoder adaptation is conditioned by the relational environment where representations are utilized.

\subsection{Are Encoder and GNN Adaptations Independent?}
\begin{figure}[htbp]
    \centering
    \includegraphics[width=0.95\linewidth]
    {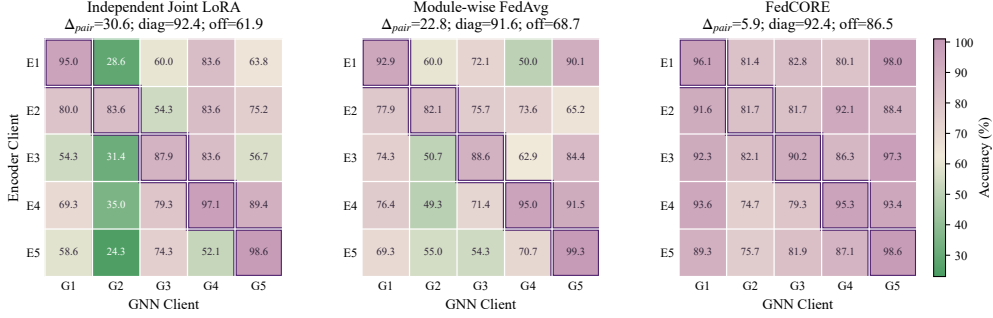}
    \caption{
    Observation 3: Independently learned module updates can exhibit
cross-client combination incompatibility.
    Cross-client Encoder--GNN pairing results on Toys.
    }
    \label{fig:obs3}
\end{figure}
\textbf{Observation 3.} \textit{Independently learned Encoder and GNN updates
can exhibit cross-client combination incompatibility.}
We examine whether independently learned Encoder and GNN updates remain
compatible across clients. We train local Encoder--GNN pairs and evaluate
cross-client combinations by replacing one module while fixing the other. As
shown in Fig.~\ref{fig:obs3}, matched Encoder--GNN pairs consistently
outperform mismatched combinations. Independent Joint LoRA suffers a pairing
gap of $30.6$, while module-wise aggregation reduces it to $22.8$. FedCORE
exhibits a substantially smaller pairing gap, demonstrating the importance of
preserving Encoder--GNN coordination during federation.

\section{Methodology}
\label{sec:method}

\subsection{Overview}

\begin{figure}[htbp]
    \centering
    \includegraphics[width=0.95\linewidth]{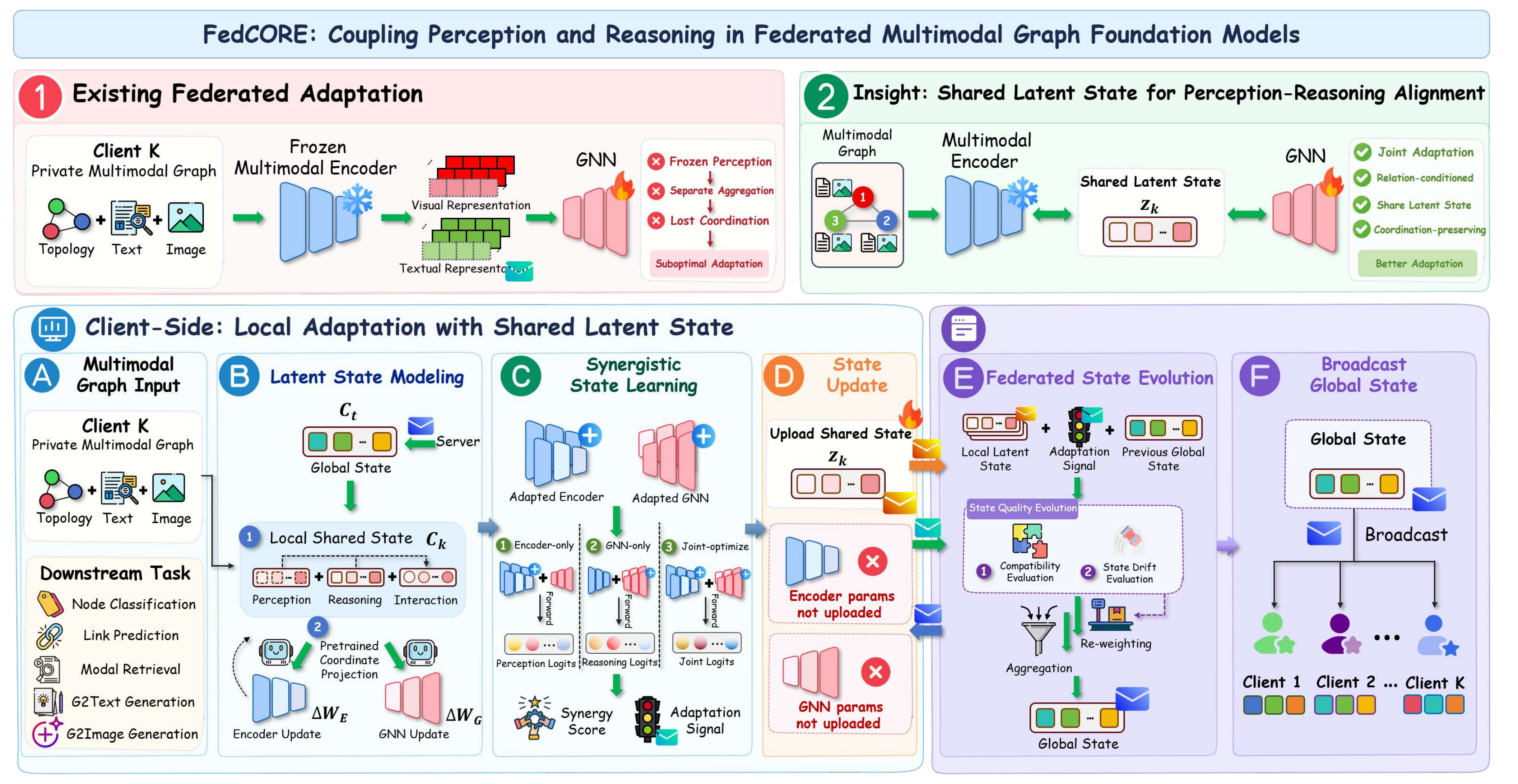}
    \caption{
    \textbf{Overview of FedCORE.}
    }
    \label{fig:framework}
\end{figure}

We present FedCORE, a latent-state adaptation framework for federated
multimodal graph foundation models. As shown in Fig.~\ref{fig:framework},
FedCORE has three components: (1) \emph{Latent State Modeling}, representing
heterogeneous Encoder--GNN adaptations via a shared state; (2)
\emph{Synergistic State Learning}, optimizing this state with task supervision
and perception--reasoning synergy; and (3) \emph{Federated State Evolution},
evolving local states across clients while preserving their coordination.
Problem formulation appears in Appendix~\ref{sec:problem}.

\subsection{Latent State Modeling}
\label{sec:state_modeling}

\textbf{Motivation.}
Obs.~2 and Obs.~3 show that Encoder and GNN adaptations are coupled: multimodal adaptation depends on relational context, while independently learned module
updates may become incompatible when combined across clients. FedCORE therefore models both through a shared latent state.

\textbf{Shared Latent State.}
Let $\mathcal J=\mathcal J_P\cup\mathcal J_R$ denote the adaptable parameter matrices, where $\mathcal J_P$ contains Encoder and fusion layers and $\mathcal J_R$ contains GNN layers. Each client $k$ maintains a matrix-valued state
$\mathbf C_k\in\mathbb R^{r\times r}$, where $r$ is the latent rank. Rather than assigning independent variables to different modules, $\mathbf C_k$ provides their common adaptation coordinate, implicitly encoding perception, reasoning, and their interaction.

\textbf{Pretrained Coordinate Construction.}
Because modules in $\mathcal J$ have heterogeneous dimensions, FedCORE constructs a fixed coordinate frame for each pretrained matrix
$\mathbf W_j^0\in\mathbb R^{d_j^{\mathrm{out}}\times d_j^{\mathrm{in}}}$
using its rank-$r$ truncated singular value decomposition:

\begin{equation}
\mathbf W_j^0
\approx
\mathbf U_j\boldsymbol{\Sigma}_j\mathbf V_j^\top,
\quad
\mathbf S_j
=
\left(
\frac{
\boldsymbol{\Sigma}_j+\epsilon\mathbf I_r
}{
\operatorname{tr}(\boldsymbol{\Sigma}_j)/r+\epsilon
}
\right)^{1/4},
\quad
\mathbf U_j^\top\mathbf U_j
=
\mathbf V_j^\top\mathbf V_j
=
\mathbf I_r .
\label{eq:pretrained_coordinate}
\end{equation}

Here,
$\mathbf U_j\in\mathbb R^{d_j^{\mathrm{out}}\times r}$,
$\mathbf V_j\in\mathbb R^{d_j^{\mathrm{in}}\times r}$,
and $\boldsymbol{\Sigma}_j,\mathbf S_j\in\mathbb R^{r\times r}$.
The constant $\epsilon>0$ ensures numerical stability. These coordinates are computed once from the pretrained model and remain frozen, leaving $\mathbf C_k$ as the client-specific adaptation variable.

\textbf{State-conditioned Adaptation Realization.}
FedCORE maps the shared state to each module through a fixed linear decoder
$\mathcal D_j:\mathbb R^{r\times r}\rightarrow
\mathbb R^{d_j^{\mathrm{out}}\times d_j^{\mathrm{in}}}$:

\begin{equation}
\Delta\mathbf W_{k,j}
=
\mathcal D_j(\mathbf C_k)
=
\alpha_j
\mathbf U_j
\mathbf S_j
\mathbf C_k
\mathbf S_j
\mathbf V_j^\top,
\quad
\alpha_j
=
\frac{
\|\mathbf W_j^0\|_F
}{
\sqrt{d_j^{\mathrm{out}}d_j^{\mathrm{in}}}+\epsilon
}.
\label{eq:state_decoder}
\end{equation}

Here, $\alpha_j$ compensates for scale differences across layers, and the adapted parameter is
$\mathbf W_{k,j}=\mathbf W_j^0+\Delta\mathbf W_{k,j}$.
Thus, heterogeneous Encoder and GNN updates retain their native parameter forms while sharing the same latent origin.

\textbf{Canonical State Recovery.}
The linear decoder admits an exact projection from each module space back to the shared state:

\begin{equation}
\mathcal P_j(\Delta\mathbf W)
=
\alpha_j^{-1}
\mathbf S_j^{-1}
\mathbf U_j^\top
\Delta\mathbf W
\mathbf V_j
\mathbf S_j^{-1},
\quad
\mathcal P_j\!\left(\mathcal D_j(\mathbf C_k)\right)
=
\mathbf C_k,
\quad
\forall j\in\mathcal J .
\label{eq:state_recovery}
\end{equation}

Hence,
$\mathcal P_i(\Delta\mathbf W_{k,i})
=
\mathcal P_j(\Delta\mathbf W_{k,j})
=
\mathbf C_k$
for any $i,j\in\mathcal J$, ensuring a common latent origin by construction.

\subsection{Synergistic State Learning}
\label{sec:state_learning}

\textbf{Motivation.}
The shared state defines how heterogeneous adaptations are represented, but must be optimized for the local task. FedCORE therefore learns $\mathbf C_k$ using downstream supervision together with perception--reasoning synergy.

\textbf{State-guided Local Adaptation.}
At round $t$, the server broadcasts the global state $\mathbf C^t$, and each participating client initializes
$\mathbf C_k^{t,0}=\mathbf C^t$.
Let $\mathcal D_P(\mathbf C_k)$ and $\mathcal D_R(\mathbf C_k)$ denote the decoded updates for modules in $\mathcal J_P$ and $\mathcal J_R$, respectively. Given local features $\mathbf X_k$, graph structure $\mathbf A_k$, labels $\mathbf Y_k$, and task head $H_{\Omega_k}$, the joint prediction is

\begin{equation}
\widehat{\mathbf Y}_k^{\,\mathrm{joint}}
=
H_{\Omega_k}
\left(
G_{\Phi_0+\mathcal D_R(\mathbf C_k)}
\left(
E_{\Theta_0+\mathcal D_P(\mathbf C_k)}
(\mathbf X_k),
\mathbf A_k
\right)
\right),
\label{eq:state_prediction}
\end{equation}

where $\Theta_0$ and $\Phi_0$ denote the pretrained Encoder-side and GNN-side parameters. Local adaptation is therefore entirely governed by the shared state rather than independent module adapters.

\textbf{Perception--Reasoning Synergy Learning.}
To evaluate whether joint adaptation provides additional benefit, we construct perception-only and reasoning-only counterfactual forward passes:

\begin{equation}
\begin{aligned}
\widehat{\mathbf Y}_k^{\,p}
&=
H_{\Omega_k}
\left(
G_{\Phi_0}
\left(
E_{\Theta_0+\mathcal D_P(\mathbf C_k)}(\mathbf X_k),
\mathbf A_k
\right)
\right),
\widehat{\mathbf Y}_k^{\,r}
=
H_{\Omega_k}
\left(
G_{\Phi_0+\mathcal D_R(\mathbf C_k)}
\left(
E_{\Theta_0}(\mathbf X_k),
\mathbf A_k
\right)
\right).
\end{aligned}
\label{eq:counterfactual_prediction}
\end{equation}

Let $\ell(\widehat{\mathbf Y},\mathbf Y)$ denote the task loss, yielding
$\ell_k^{\mathrm{joint}}$, $\ell_k^{p}$, and $\ell_k^{r}$ for the three predictions.
To prevent synergy from being artificially increased by degrading isolated branches, FedCORE treats their losses as detached references. With $\operatorname{sg}(\cdot)$ denoting stop-gradient and $[x]_+=\max(x,0)$,

\begin{equation}
b_k
=
\operatorname{sg}
\left(
\min\{\ell_k^{p},\ell_k^{r}\}
\right),
\quad
\mathcal L_k^{\mathrm{syn}}
=
\left[
m+\ell_k^{\mathrm{joint}}-b_k
\right]_+,
\quad
\mathcal L_k
=
\ell_k^{\mathrm{joint}}
+
\lambda_c\mathcal L_k^{\mathrm{syn}}
+
\frac{\lambda_r}{2}
\|\mathbf C_k\|_F^2 ,
\label{eq:synergy_objective}
\end{equation}

where $m\geq0$ is the synergy margin and $\lambda_c,\lambda_r\geq0$ control synergy and state regularization. This objective encourages joint adaptation to outperform the stronger isolated branch without degrading either reference branch.

\textbf{Latent Signal Integration.}
Since module gradients lie in heterogeneous parameter spaces, FedCORE maps them into the shared state space. For module $j$, let
$\mathbf G_{k,j}^{W}
=
\nabla_{\mathbf W_{k,j}}\mathcal L_k$.
The adjoint of the decoder gives

\begin{equation}
\mathbf g_{k,j}
=
\mathcal D_j^{*}
\left(
\mathbf G_{k,j}^{W}
\right)
=
\alpha_j
\mathbf S_j
\mathbf U_j^\top
\mathbf G_{k,j}^{W}
\mathbf V_j
\mathbf S_j,
\quad
\mathbf g_k^{p}
=
\sum_{j\in\mathcal J_P}\mathbf g_{k,j},
\quad
\mathbf g_k^{r}
=
\sum_{j\in\mathcal J_R}\mathbf g_{k,j}.
\label{eq:latent_gradient}
\end{equation}

The resulting state gradient and local update are

\begin{equation}
\mathbf{g}_k^e
=
\mathbf{g}_{k}^{p,e}
+
\mathbf{g}_{k}^{r,e}
+
\lambda_r\mathbf{C}_k^e,
\qquad
\mathbf{C}_k^{t,e+1}
=
\mathbf{C}_k^{t,e}
-
\eta
\mathbf{P}_{L,k}^e
\mathbf{g}_k^e
\mathbf{P}_{R,k}^e.
\tag{8}
\end{equation}

where $e\in\{0,\ldots,E-1\}$ indexes the $E$ local steps and $\eta>0$ is the learning rate. Perception and reasoning signals thus jointly update a single client state. The preconditioners and their local second-moment statistics are defined in
Appendix~\ref{app:theory}.

\subsection{Federated State Evolution}
\label{sec:evolution}

\textbf{Motivation.}
Local training produces client-specific states encoding coordinated perception--reasoning adaptation. FedCORE evolves the shared states directly, thereby enforcing a common parameterization for the decoded Encoder and GNN updates.

\textbf{State Quality-aware Evolution.}
After $E$ local steps, client $k$ obtains $\mathbf C_k^{t,E}$. We summarize its perception--reasoning coordination by averaging the latent signals along the local trajectory:

\begin{equation}
\overline{\mathbf g}_k^{p}
=
\frac{1}{E}
\sum_{e=0}^{E-1}
\mathbf g_{k,e}^{p},
\quad
\overline{\mathbf g}_k^{r}
=
\frac{1}{E}
\sum_{e=0}^{E-1}
\mathbf g_{k,e}^{r},
\quad
\Gamma_k
=
\frac{
\left\langle
\overline{\mathbf g}_k^{p},
\overline{\mathbf g}_k^{r}
\right\rangle_F
}{
\|
\overline{\mathbf g}_k^{p}
\|_F
\|
\overline{\mathbf g}_k^{r}
\|_F+\epsilon
}.
\label{eq:coordination_score}
\end{equation}

Here,
$\langle\mathbf A,\mathbf B\rangle_F=\operatorname{tr}(\mathbf A^\top\mathbf B)$
is the Frobenius inner product, and larger $\Gamma_k$ indicates stronger perception--reasoning agreement. Each client uploads only $\mathbf C_k^{t,E}$ and the scalar $\Gamma_k$; module parameters and gradients remain local.

The server combines this coordination score with normalized state drift to determine the evolution weight:

\begin{equation}
D_k^t
=
\frac{
\|
\mathbf C_k^{t,E}-\mathbf C^t
\|_F^2
}{
\|
\mathbf C^t
\|_F^2+\epsilon
},
\qquad
\omega_k^t
=
\frac{
p_k^t
\exp\!\left(
\tau\Gamma_k-\xi D_k^t
\right)
}{
\sum_{i\in\mathcal S_t}
p_i^t
\exp\!\left(
\tau\Gamma_i-\xi D_i^t
\right)
},
\quad
p_k^t
=
\frac{n_k}{\sum_{i\in\mathcal S_t}n_i},
\label{eq:evolution_weight}
\end{equation}

where $\mathcal S_t$ denotes participating clients, $n_k$ is the number of local training instances, and $\tau,\xi\geq0$ control coordination and drift. This weighting favors coordinated states while suppressing excessive deviation from the previous global state.

\textbf{Latent State Evolution.}
FedCORE combines the quality-aware client states with the previous global state:

\begin{equation}
\mathbf C^{t+1}
=
(1-\mu)
\sum_{k\in\mathcal S_t}
\omega_k^t
\mathbf C_k^{t,E}
+
\mu\mathbf C^t,
\qquad
\mathcal D_j(\mathbf C^{t+1})
=
(1-\mu)
\sum_{k\in\mathcal S_t}
\omega_k^t
\mathcal D_j(\mathbf C_k^{t,E})
+
\mu\mathcal D_j(\mathbf C^t),
\label{eq:global_state_update}
\end{equation}

where $\mu\in[0,1)$ controls temporal persistence. By linearity of $\mathcal D_j$, state evolution remains closed under the coupled adaptation family, so all global module updates continue to originate from the same $\mathbf C^{t+1}$. The updated state is then broadcast for the next communication round.

\section{Experiments}
\label{sec:experiments}

We conduct comprehensive experiments to evaluate FedCORE and address four
research questions. \textbf{Q1: Overall Performance.} Does FedCORE consistently
outperform existing federated multimodal graph learning methods across
downstream tasks? \textbf{Q2: Ablation Study.} How does each design contribute
to FedCORE's effectiveness? \textbf{Q3: Robustness.} Does FedCORE remain
effective under stronger client heterogeneity and structural perturbations?
\textbf{Q4: Efficiency and Scalability.} Can FedCORE achieve these gains with
reasonable communication, computation, and memory overhead?

\subsection{Experimental Setup}

We evaluate FedCORE on multimodal graph benchmarks spanning node classification,
link prediction, multimodal retrieval, graph-to-text generation, and
graph-to-image generation. Toys and Grocery derive from Amazon review data,
whereas Bili Music, DY, QB, and Bili Cartoon come from NineRec
\citep{zhang2024ninerec}. Following MM-OpenFGL \citep{li2026mm}, we
compare FedCORE with representative federated, federated graph foundation, and
multimodal federated graph methods, including FedAvg
\citep{mcmahan2017communication}, FedGFM+ \citep{zhu2026towards}, STAGE
\citep{chen2026stage}, FedGAMMA \citep{li2026fedgamma}, and federated PLANET
\citep{liu2026toward}, plus two adaptation controls: Frozen Encoder + GNN and
Independent Joint LoRA \citep{hu2021lora}. All methods share client
partitions, pretrained backbones, and evaluation protocols. Detailed settings
appear in Appendix~\ref{app:setup}.

\subsection{Overall Performance (Q1)}

\begin{table*}[htbp]
    \centering
    \caption{Overall performance comparison. Best results are in 
    \colorbox[HTML]{DADADA}{\textbf{bold}} and second-best are \underline{underlined}.}
    \vspace{-2pt}

    \scriptsize
    \setlength{\tabcolsep}{3.5pt}
    \renewcommand{\arraystretch}{1.05}
    \begin{tabular}{l|c c | c c | c c | c | c}

    \specialrule{1.5pt}{1.5pt}{1.5pt}

    \multirow{2}{*}{\textbf{Method}}
    & \multicolumn{2}{c|}{\begin{tabular}{@{}c@{}}\textbf{Node Classification} \\ \textbf{(Acc)}\end{tabular}}
    & \multicolumn{2}{c|}{\begin{tabular}{@{}c@{}}\textbf{Link Prediction} \\ \textbf{(AUC)}\end{tabular}}
    & \multicolumn{2}{c|}{\begin{tabular}{@{}c@{}}\textbf{Modal Retrieval} \\ \textbf{(R@5)}\end{tabular}}
    & \begin{tabular}{@{}c@{}}\textbf{G2Text} \\ \textbf{(ROUGE-L)}\end{tabular}
    & \begin{tabular}{@{}c@{}}\textbf{G2Image} \\ \textbf{(CLIP-S)}\end{tabular} \\

    \cmidrule(lr){2-3}
    \cmidrule(lr){4-5}
    \cmidrule(lr){6-7}
    \cmidrule(lr){8-8}
    \cmidrule(lr){9-9}

    & \textbf{Toys} 
    & \textbf{Grocery} 
    & \textbf{Bili Music} 
    & \textbf{DY} 
    & \textbf{QB} 
    & \textbf{Bili Cartoon} 
    & \textbf{Flickr30k} 
    & \textbf{SemArt} \\

    \midrule

    FedGAMMA &
    $76.37_{\scriptstyle \pm 0.65}$ &
    $73.73_{\scriptstyle \pm 0.54}$ &
    $75.87_{\scriptstyle \pm 0.91}$ &
    $72.25_{\scriptstyle \pm 0.26}$ &
    $86.52_{\scriptstyle \pm 2.44}$ &
    $80.22_{\scriptstyle \pm 0.78}$ &
    $31.29_{\scriptstyle \pm 0.09}$ &
    $70.87_{\scriptstyle \pm 0.23}$ \\

    STAGE &
    $79.19_{\scriptstyle \pm 0.02}$ &
    $81.04_{\scriptstyle \pm 0.42}$ &
    $69.41_{\scriptstyle \pm 1.74}$ &
    $69.19_{\scriptstyle \pm 0.49}$ &
    \underline{$94.91_{\scriptstyle \pm 3.65}$} &
    $79.77_{\scriptstyle \pm 1.39}$ &
    $30.87_{\scriptstyle \pm 0.47}$ &
    \underline{$71.86_{\scriptstyle \pm 0.53}$} \\

    FedGFM+ &
    $80.21_{\scriptstyle \pm 0.28}$ &
    $83.76_{\scriptstyle \pm 0.15}$ &
    \underline{$76.24_{\scriptstyle \pm 0.02}$} &
    \underline{$74.30_{\scriptstyle \pm 0.11}$} &
    $92.06_{\scriptstyle \pm 0.83}$ &
    \underline{$84.57_{\scriptstyle \pm 4.38}$} &
    $31.09_{\scriptstyle \pm 0.31}$ &
    $71.58_{\scriptstyle \pm 0.26}$ \\

    Fed-PLANET &
    \underline{$80.58_{\scriptstyle \pm 0.07}$} &
    \underline{$83.93_{\scriptstyle \pm 0.48}$} &
    $65.29_{\scriptstyle \pm 0.90}$ &
    $64.08_{\scriptstyle \pm 0.70}$ &
    $89.89_{\scriptstyle \pm 2.64}$ &
    $81.20_{\scriptstyle \pm 5.54}$ &
    \underline{$31.35_{\scriptstyle \pm 0.23}$} &
    $71.38_{\scriptstyle \pm 0.21}$ \\

    \shortstack[l]{Frozen Enc-GNN} &
    $78.93_{\scriptstyle \pm 0.14}$ &
    $80.63_{\scriptstyle \pm 0.08}$ &
    $74.50_{\scriptstyle \pm 0.70}$ &
    $72.60_{\scriptstyle \pm 0.21}$ &
    $86.45_{\scriptstyle \pm 2.56}$ &
    $76.66_{\scriptstyle \pm 5.00}$ &
    $31.07_{\scriptstyle \pm 0.35}$ &
    $70.55_{\scriptstyle \pm 0.09}$ \\

    \shortstack[l]{Joint LoRA} &
    $78.72_{\scriptstyle \pm 0.04}$ &
    $80.59_{\scriptstyle \pm 0.55}$ &
    $73.64_{\scriptstyle \pm 0.69}$ &
    $72.41_{\scriptstyle \pm 1.64}$ &
    $80.49_{\scriptstyle \pm 1.95}$ &
    $81.49_{\scriptstyle \pm 2.27}$ &
    $30.53_{\scriptstyle \pm 0.07}$ &
    $70.21_{\scriptstyle \pm 0.12}$ \\

    \cellcolor[HTML]{DADADA}\textbf{Ours} &
    \cellcolor[HTML]{DADADA}$\textbf{81.72}_{\scriptstyle \pm \textbf{0.10}}$ &
    \cellcolor[HTML]{DADADA}$\textbf{84.68}_{\scriptstyle \pm \textbf{0.18}}$ &
    \cellcolor[HTML]{DADADA}$\textbf{79.50}_{\scriptstyle \pm \textbf{0.94}}$ &
    \cellcolor[HTML]{DADADA}$\textbf{77.69}_{\scriptstyle \pm \textbf{0.32}}$ &
    \cellcolor[HTML]{DADADA}$\textbf{95.05}_{\scriptstyle \pm \textbf{1.61}}$ &
    \cellcolor[HTML]{DADADA}$\textbf{90.73}_{\scriptstyle \pm \textbf{2.34}}$ &
    \cellcolor[HTML]{DADADA}$\textbf{33.31}_{\scriptstyle \pm \textbf{0.59}}$ &
    \cellcolor[HTML]{DADADA}$\textbf{73.24}_{\scriptstyle \pm \textbf{0.27}}$ \\

    \specialrule{1.3pt}{2.0pt}{1.0pt}

    \end{tabular}

    \vspace{-3pt}
    \label{table:overall}
\end{table*}

For Q1, we compare FedCORE with representative federated graph, multimodal
graph, and GFM baselines across diverse tasks. Table~\ref{table:overall} shows
that FedCORE consistently performs best, exceeding the strongest non-FedCORE
baseline by 2.27\% on average and up to 6.16\% on Bili Cartoon. FedCORE also
surpasses Independent Joint LoRA by 5.98\% on average. These results show that
the complete FedCORE parameterization outperforms independently maintained
Encoder and GNN adapters in our implementation.

\subsection{Ablation Study (Q2)}
\label{sec:ablation}

\begin{table*}[htbp]
\centering

\caption{Ablation study.
Best results are in \colorbox[HTML]{DADADA}{\textbf{bold}} and second-best
results are \underline{underlined}.}
\label{tab:ablation}

\vspace{-2pt}

\scriptsize
\setlength{\tabcolsep}{4.5pt}
\renewcommand{\arraystretch}{1.05}

\newcolumntype{M}[1]{%
  >{\footnotesize\raggedright\arraybackslash}m{#1}
}

\begin{tabular}{M{4.6cm}|c|c|c|c|c}

\specialrule{1.5pt}{1.5pt}{1.5pt}

\multirow{2}{*}{\small\textbf{Variant}}
&
\begin{tabular}{@{}c@{}}
\textbf{Node Classification}\\
\textbf{(Acc)}
\end{tabular}
&
\begin{tabular}{@{}c@{}}
\textbf{Link Prediction}\\
\textbf{(AUC)}
\end{tabular}
&
\begin{tabular}{@{}c@{}}
\textbf{Modal Retrieval}\\
\textbf{(R@5)}
\end{tabular}
&
\begin{tabular}{@{}c@{}}
\textbf{G2Text}\\
\textbf{(ROUGE-L)}
\end{tabular}
&
\begin{tabular}{@{}c@{}}
\textbf{G2Image}\\
\textbf{(CLIP-S)}
\end{tabular}
\\

\cmidrule(lr){2-2}
\cmidrule(lr){3-3}
\cmidrule(lr){4-4}
\cmidrule(lr){5-5}
\cmidrule(lr){6-6}

&
\textbf{Toys}
&
\textbf{DY}
&
\textbf{QB}
&
\textbf{Flickr30k}
&
\textbf{SemArt}
\\

\midrule

{\footnotesize\textbf{Full FedCORE}}
&
\cellcolor[HTML]{DADADA}
$\textbf{81.67}_{\scriptstyle \pm \textbf{0.05}}$
&
\cellcolor[HTML]{DADADA}
$\textbf{77.72}_{\scriptstyle \pm \textbf{0.32}}$
&
\cellcolor[HTML]{DADADA}
$\textbf{95.28}_{\scriptstyle \pm \textbf{0.61}}$
&
\cellcolor[HTML]{DADADA}
$\textbf{33.21}_{\scriptstyle \pm \textbf{0.76}}$
&
\cellcolor[HTML]{DADADA}
$\textbf{73.26}_{\scriptstyle \pm \textbf{0.21}}$
\\

\footnotesize w/o Shared Latent State
&
$79.96_{\scriptstyle \pm 0.21}$
&
$74.86_{\scriptstyle \pm 0.43}$
&
$92.94_{\scriptstyle \pm 0.46}$
&
$30.71_{\scriptstyle \pm 0.94}$
&
$71.75_{\scriptstyle \pm 0.36}$
\\

\footnotesize w/o SVD Bases
&
$80.87_{\scriptstyle \pm 0.12}$
&
$75.93_{\scriptstyle \pm 0.32}$
&
$94.18_{\scriptstyle \pm 0.73}$
&
$32.48_{\scriptstyle \pm 0.85}$
&
$72.59_{\scriptstyle \pm 0.29}$
\\

\makebox[5.4cm][l]{\footnotesize w/o Geometry-Balanced Update}
&
\underline{$81.24_{\scriptstyle \pm 0.08}$}
&
$76.31_{\scriptstyle \pm 0.29}$
&
$94.76_{\scriptstyle \pm 0.66}$
&
\underline{$32.91_{\scriptstyle \pm 0.75}$}
&
\underline{$72.94_{\scriptstyle \pm 0.24}$}
\\

\footnotesize Indep. Adapters + FedAvg
&
$80.51_{\scriptstyle \pm 0.06}$
&
\underline{$76.66_{\scriptstyle \pm 0.27}$}
&
\underline{$95.05_{\scriptstyle \pm 0.15}$}
&
$32.19_{\scriptstyle \pm 0.40}$
&
$72.24_{\scriptstyle \pm 0.21}$
\\

\specialrule{1.3pt}{2.0pt}{1.0pt}

\end{tabular}

\vspace{-3pt}

\end{table*}
To answer Q2, we ablate FedCORE's key designs in Table~\ref{tab:ablation}.
We replace the shared latent state with independent Encoder--GNN adapters,
remove the fixed SVD bases, replace the geometry-balanced update with ordinary
gradient descent, and replace shared core aggregation with independent adapter
aggregation. These variants assess the contributions of shared adaptation
representation, coordinate transport, geometry-balanced optimization, and
aggregation-consistent federation.

\begin{figure*}[htbp]
    \centering
    \includegraphics[width=0.95\linewidth]
    {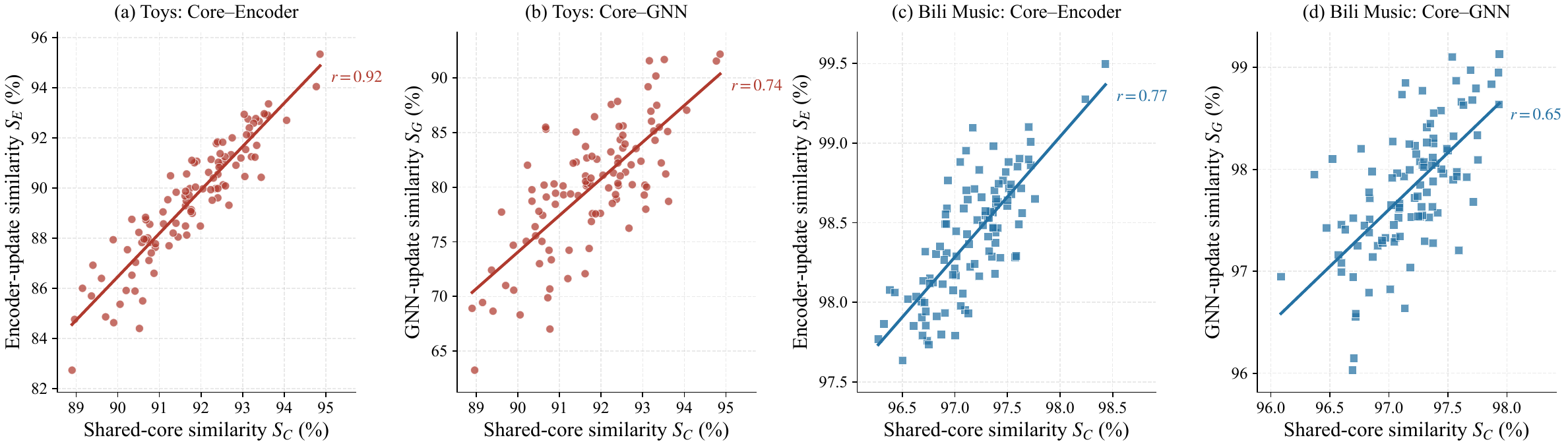}
    \caption{
        \textbf{Relational core--update correspondence.}
Each point is one client pair. The x-axis shows shared-core similarity, and the y-axis shows Encoder/GNN update similarity.
    }
    \label{fig:core_correspondence}
\end{figure*}

Building on these
ablations, we further test whether the learned shared core captures a common
perception--reasoning adaptation state. Fig.~\ref{fig:core_correspondence}
measures the relationship between pairwise shared-core similarity and decoded
Encoder and GNN update similarity on Toys and Bili Music. Positive correlations
indicate that clients with similar shared cores tend to receive similar
perception and reasoning adaptations.

\subsection{Robustness Analysis (Q3)}
\label{sec:robustness}

To answer Q3, we evaluate FedCORE under stronger client heterogeneity, graph fragmentation, hyperparameter variations, and imperfect multimodal inputs.

\begin{figure*}[htbp]
\centering

\begin{minipage}[t]{0.5\textwidth}
\centering
\includegraphics[width=\linewidth]
{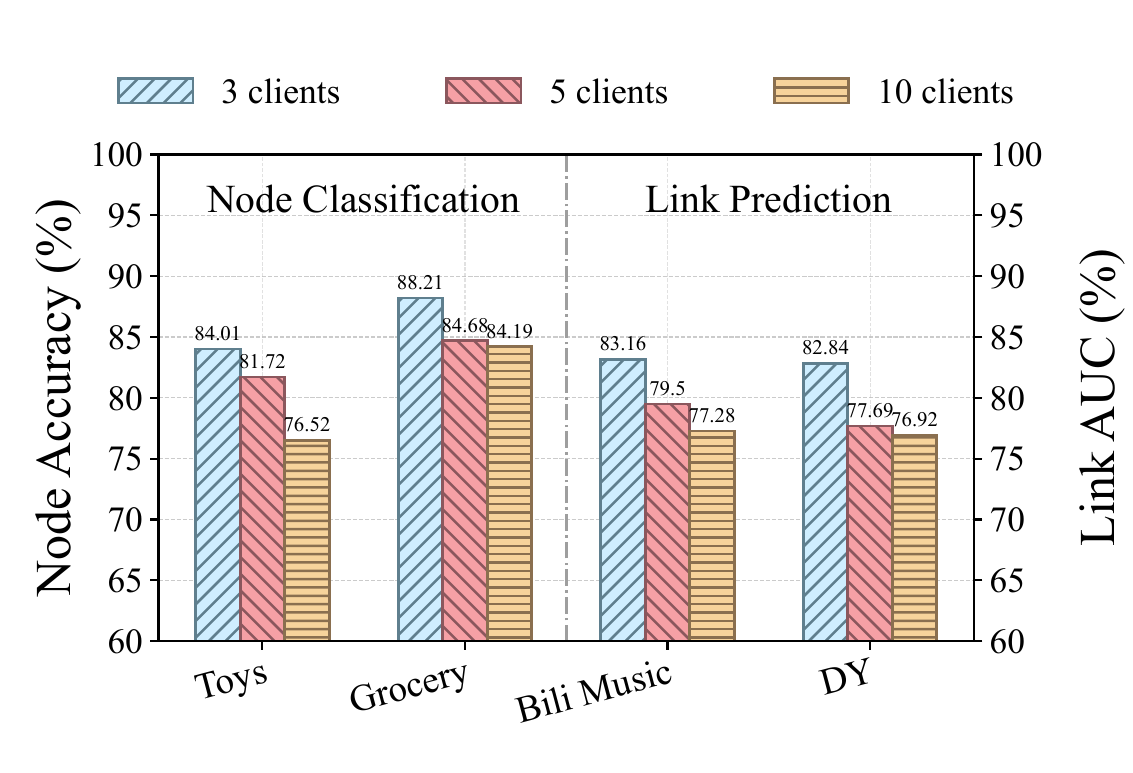}\\[-2pt]
\small (a) Client fragmentation.
\end{minipage}
\hfill
\begin{minipage}[t]{0.4\textwidth}
\centering
\includegraphics[width=\linewidth]
{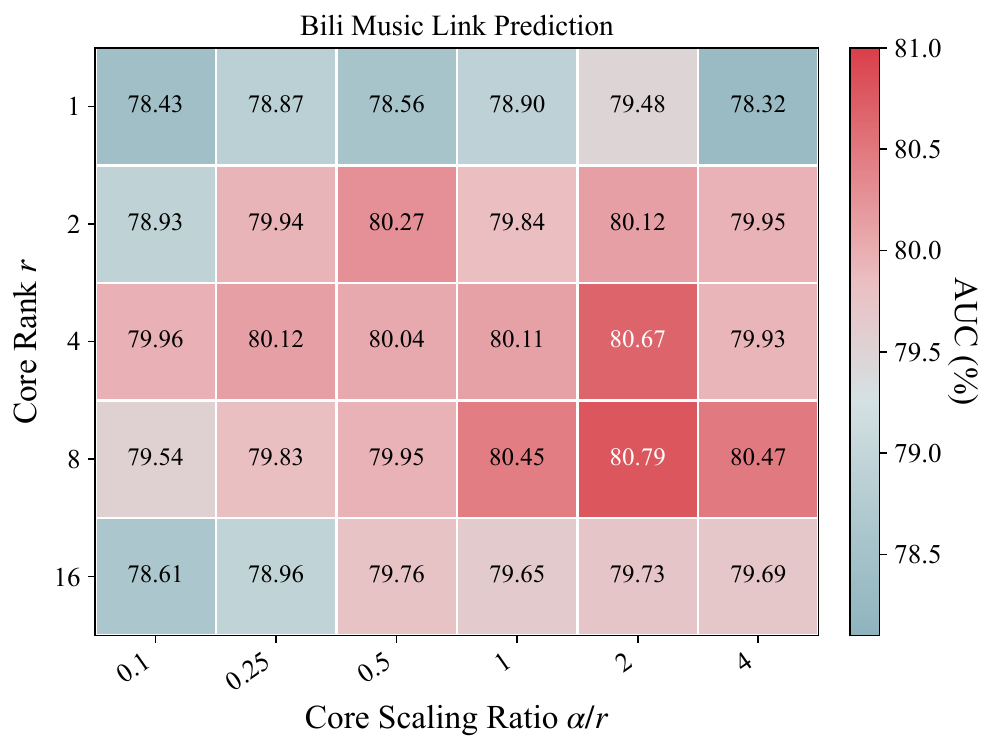}\\[-2pt]
\small (b) Hyperparameter sensitivity.
\end{minipage}

\vspace{-4pt}
\caption{
\textbf{Robustness to client fragmentation and hyperparameter choices.}
}
\label{fig:robustness_scalability_sensitivity}
\end{figure*}
\textbf{Scalability to Client Fragmentation.}
We vary the number of clients as $K\in\{3,5,10\}$ while keeping the complete dataset fixed. Larger $K$ leads to smaller and more fragmented local graphs, making local adaptation and global coordination more difficult. Fig.~\ref{fig:robustness_scalability_sensitivity}(a) reports the performance of FedCORE on representative node classification and link prediction tasks.

\textbf{Hyperparameter Sensitivity.}
We study the latent state rank $r$ and the core scaling ratio $\alpha/r$, which control the capacity and magnitude of the shared latent state. Fig.~\ref{fig:robustness_scalability_sensitivity}(b) presents the sensitivity heatmap on Bili Music link prediction. The broad high-performing region indicates that FedCORE is not sensitive to precise hyperparameter choices.

\textbf{Robustness to Modality Corruption and Missing Modalities.}
We further evaluate robustness to degraded modality information on Grocery node classification. Fig.~\ref{fig:modality_robustness} reports performance as the corruption and missing-modality ratios increase, examining whether coupled Encoder--GNN adaptation can better exploit the remaining semantic and relational information.

\begin{figure}[htbp]
    \centering
    \includegraphics[width=\linewidth]
    {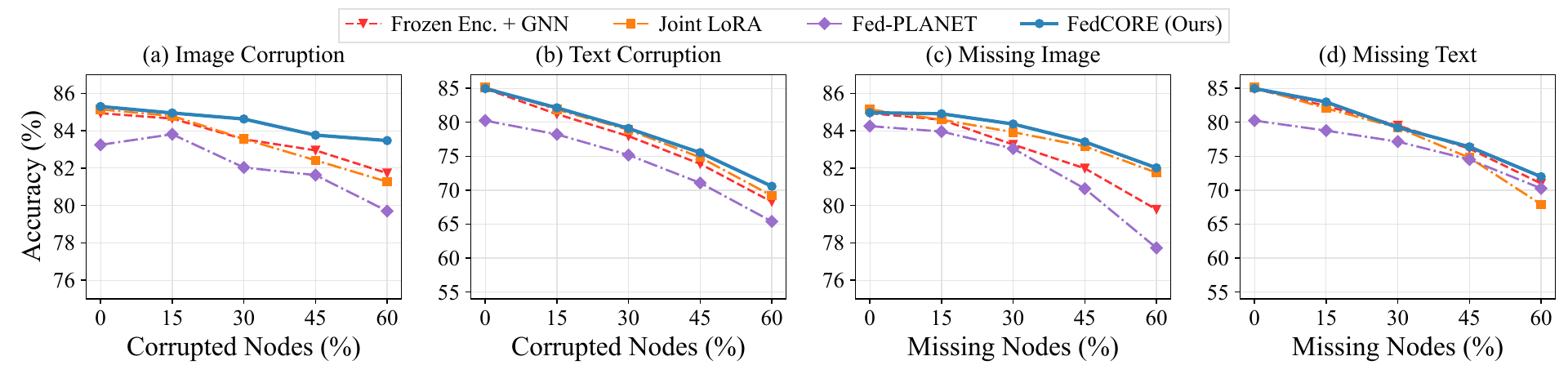}
    
    \caption{
        \textbf{Robustness to modality corruption and missing modalities.}
    }
    \label{fig:modality_robustness}
\end{figure}

We also evaluate FedCORE under imperfect multimodal inputs and client
heterogeneity. Additional results on controlled feature and topology
heterogeneity are provided in Appendix~\ref{app:additional_robustness}.

\subsection{Efficiency and Scalability Analysis (Q4)}
\label{sec:efficiency}

To answer Q4, we evaluate whether FedCORE achieves coupled adaptation with reasonable system overhead and remains scalable as more model components are adapted.

\textbf{Communication and Resource Efficiency.}
Table~\ref{tab:efficiency} compares FedCORE with representative baselines in
trainable parameters, communication payload, FLOPs, GPU memory, and training
time. We count one upload and one download of federated adaptation parameters
per participating client per round. FedCORE communicates
$5.12\times10^{2}$ scalars per round, compared with
$4.64\times10^{5}$ scalars for Independent Joint LoRA, a 99.89\% reduction in
per-round communication. Appendix~\ref{app:efficiency_scalability} reports
additional convergence and adaptation-scale analyses.

\begin{table*}[htbp]
\centering
\caption{Communication and resource efficiency of FedCORE and baseline methods.}
\label{tab:efficiency}
\renewcommand{\arraystretch}{1.15}
\setlength{\tabcolsep}{6pt}

\resizebox{\textwidth}{!}{%
\begin{tabular}{lccccc}
\toprule
\multicolumn{1}{c}{\textbf{Method}}
& \makecell{\textbf{Trainable}\\\textbf{Params.}}
& \makecell{\textbf{Comm. Payload}\\\textbf{(Scalars/Round)}}
& \makecell{\textbf{Total Ops}\\\textbf{(FLOPs)}}
& \makecell{\textbf{Space Cost}\\\textbf{(MB)}}
& \makecell{\textbf{Time}\\\textbf{(s/round)}} \\
\midrule

FedGAMMA
& $2.32 \times 10^{5}$
& $4.64 \times 10^{5}$
& $2.69 \times 10^{11}$
& $\mathbf{242.0}$
& $2.933$ \\

STAGE
& $2.49 \times 10^{5}$
& $4.98 \times 10^{5}$
& $\mathbf{2.69 \times 10^{10}}$
& $469.6$
& $1.795$ \\

Fed-PLANET
& $8.43 \times 10^{5}$
& $1.686 \times 10^{6}$
& $2.73 \times 10^{11}$
& $528.0$
& $2.088$ \\

Frozen Encoder + Trainable GNN + FedAvg
& $2.49 \times 10^{5}$
& $4.98 \times 10^{5}$
& $\mathbf{2.69 \times 10^{10}}$
& $443.6$
& $\mathbf{0.676}$ \\

Independent Joint LoRA + FedAvg
& $2.32 \times 10^{5}$
& $4.64 \times 10^{5}$
& $2.90 \times 10^{10}$
& $438.8$
& $0.752$ \\

\rowcolor{gray!15}
FedCORE (Ours)
& $\mathbf{2.56 \times 10^{2}}$
& $\mathbf{5.12 \times 10^{2}}$
& $3.12 \times 10^{10}$
& $441.3$
& $1.155$ \\

\bottomrule
\end{tabular}%
}
\end{table*}

\section{Conclusion}
\label{sec:conclusion}

In this work, we revisit federated adaptation of multimodal graph foundation models and show that Encoder adaptation is necessary, sensitive to graph relations, and coupled with GNN adaptation. Based on these findings, we propose FedCORE, which models perception and graph reasoning through a shared latent state and preserves their coordination during local optimization and federated aggregation. Experiments across diverse multimodal graph tasks demonstrate that FedCORE consistently outperforms frozen-Encoder and independent adaptation strategies, highlighting the importance of coordinated perception--reasoning adaptation in federated multimodal GFMs.

\subsection*{AI use statement}

Generative AI tools were used during the preparation of this manuscript for limited purposes, including language refinement, literature retrieval and discovery, and research ideation support. Specifically, AI tools assisted with improving writing clarity, organizing related work, exploring potential research directions, and discussing possible methodological designs and experimental considerations.

\subsection*{Reproducibility Statement}

We provide an anonymized reference implementation, experiment
configurations, and reproducibility checks at
\url{https://anonymous.4open.science/r/FedCORE-A34E/}.
The main paper describes the empirical diagnostic protocols, the
FedCORE parameterization and optimization procedure, and the
evaluation settings and results. The appendix provides additional
protocol details, theoretical derivations, robustness analyses, and
implementation settings. The released code excludes datasets and
pretrained checkpoints, but documents the required data format,
external model dependencies, and instructions for reproducing the
reported experimental settings.

\bibliography{iclr2027_conference}
\bibliographystyle{iclr2027_conference}

\appendix
\section{Appendix}
\subsection{Problem Formulation}
\label{sec:problem}

We consider federated multimodal graph foundation models consisting of a pretrained multimodal Encoder $E_{\Theta}$ and a GNN $G_{\Phi}$. Each client $k$ owns a private multimodal graph $\mathcal G_k=(\{\mathbf X_k^m\}_{m\in\mathcal M},\mathbf A_k,\mathbf Y_k)$, where $\mathcal M$ denotes the modality set, $\mathbf A_k$ the graph structure, and $\mathbf Y_k$ the task labels. Raw data remain local, and downstream prediction is performed as
\begin{equation}
\widehat{\mathbf Y}_k
=
H_{\Omega_k}\!\left(
G_{\Phi}\!\left(
E_{\Theta}(\{\mathbf X_k^m\}_{m\in\mathcal M}),
\mathbf A_k
\right)\right).
\end{equation}
Conventional federated adaptation aggregates Encoder and GNN updates separately across clients, without explicitly preserving their locally learned coordination. We therefore study how to jointly represent and federatively evolve perception--reasoning adaptations while keeping client privacy.

\subsection{Structural Randomization Analysis}
\label{app:structural_randomization}

To complement the parameter-level analysis in Sec.~\ref{sec:empirical}, we further study how the methods behave when graph relations are gradually corrupted. We construct graphs with different structural integrity ratios $\rho$, ranging from a fully randomized graph ($\rho=0$) to the original graph ($\rho=1$), while keeping node contents and supervision unchanged. This type of community-based subgraph construction is consistent with
standard subgraph federated learning protocols
\citep{wang2022federatedscope,zhang2021subgraph}.

\begin{figure}[t]
    \centering
    \includegraphics[width=\linewidth]
    {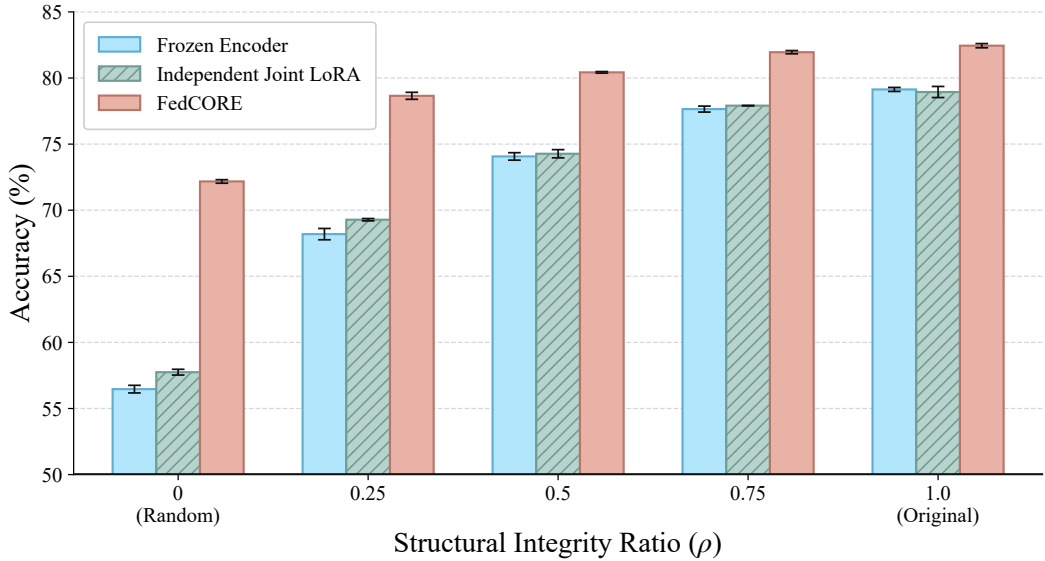}
    \caption{
    \textbf{Performance under progressive structural randomization.}
    Absolute accuracy of Frozen Encoder, Independent Joint LoRA, and FedCORE as the structural integrity ratio $\rho$ increases from a randomized graph to the original graph.
    }
    \label{fig:structural_randomization}
\end{figure}

As shown in Fig.~\ref{fig:structural_randomization}, all methods improve as more graph structure is preserved, and FedCORE consistently achieves higher accuracy across all integrity levels. This result complements Observation~2 by showing the performance-level effect of structural corruption, whereas the main empirical study examines how relational changes directly alter Encoder adaptation.

\subsection{Controlled Protocols for Empirical Study}
\label{app:empirical_setup}

This section details the controlled protocols underlying the empirical observations in Sec.~\ref{sec:empirical}. The studies are designed as a progressive diagnostic chain rather than independent performance comparisons. This diagnostic setting is related to prior work on subgraph federated
learning, which studies missing cross-client neighbors and structural
heterogeneity across local subgraphs
\citep{zhang2021subgraph,xie2021federated,baek2023personalized}. Specifically, we examine whether graph-side capacity can substitute for Encoder adaptation, whether graph-relational interventions alter the preferred Encoder update, and whether the resulting Encoder and GNN adaptations remain
independently composable. Unless otherwise specified, all experiments use identical node contents, labels, data splits, pretrained backbones, and optimization budgets, with only the investigated factor being changed.

\subsubsection{Diagnosing the Limit of Graph-Side Adaptation}

The first study examines whether the benefit of Encoder adaptation can be reproduced simply by increasing GNN capacity. We conduct the analysis on Toys and QB and keep the pretrained multimodal Encoder completely frozen. For each GNN family $g\in\{\mathrm{GCN},\mathrm{GraphSAGE},\mathrm{GAT}\}$, we construct a sequence of models with increasing hidden capacity and record their trainable parameter budgets $P_g$. Their performance is compared with Encoder-LoRA and full fine-tuning under the same data partition and training protocol.

Let $\mathcal{S}_{g}(P)$ denote the performance of a frozen-Encoder model with GNN family $g$ and approximately $P$ trainable parameters, and let $\mathcal{S}_{E}$ denote the performance obtained when the Encoder is adapted. We characterize the residual representation gap at budget $P$ as

\begin{equation}
\Delta_{\mathrm{cap}}(P)
=
\mathcal{S}_{E}
-
\max_{g\in\mathcal{G}}
\mathcal{S}_{g}(P),
\quad
\mathcal{G}
=
\{\mathrm{GCN},\mathrm{GraphSAGE},\mathrm{GAT}\}.
\end{equation}

A positive $\Delta_{\mathrm{cap}}(P)$ across sufficiently expressive GNN configurations indicates that additional graph-side capacity cannot fully substitute for changing the multimodal representation itself. Importantly, this study does not require GNN performance to strictly saturate. The diagnostic criterion is whether stronger GNNs can \emph{consistently close} the gap to Encoder-adapted models. As observed in Fig.~\ref{fig:obs1}, increasing GNN capacity improves performance, particularly on QB, yet the frozen-Encoder configurations remain consistently constrained relative to Encoder adaptation. This supports the premise that federated multimodal GFMs should expose adaptation capacity on both perception and graph-reasoning sides.

\subsubsection{Controlled Relational Intervention}

The second study examines whether Encoder adaptation is affected by the graph
on which the representations are consumed, beyond node-level multimodal
content. To examine this effect, we construct a family of controlled graphs

\begin{equation}
\mathcal{G}^{(\tau)}
=
\left(
\mathcal{V},
\mathcal{A}^{(\tau)},
\mathbf{X},
\mathbf{Y}
\right),
\quad
\tau\in[0,1],
\end{equation}

where node identities, multimodal contents $\mathbf{X}$, labels $\mathbf{Y}$, and train/validation/test splits remain fixed, while only the relational structure $\mathcal{A}^{(\tau)}$ is changed. Relational interventions are generated through constrained edge rewiring with approximately preserved node-degree profiles,

\begin{equation}
\left\|
\mathbf{d}^{(\tau)}
-
\mathbf{d}^{(0)}
\right\|_1
\approx 0,
\end{equation}

so that the intervention primarily changes \emph{who is connected to whom} rather than trivially changing graph density or node exposure.

We quantify the semantic consequence of a structural intervention using the change of the neighborhood-conditioned semantic profile. Let $\mathbf{z}_i$ denote the fixed multimodal representation of node $i$ and

\begin{equation}
\mathbf{h}_i^{(\tau)}
=
\frac{
\sum_{j\in\mathcal{N}_i^{(\tau)}}\mathbf{z}_j
}{
|\mathcal{N}_i^{(\tau)}|+\epsilon
}.
\end{equation}

The relational semantic shift is measured as

\begin{equation}
D_{\mathrm{rel}}^{(\tau)}
=
\frac{1}{|\mathcal{V}|}
\sum_{i\in\mathcal{V}}
\left[
1-
\cos
\left(
\mathbf{h}_i^{(\tau)},
\mathbf{h}_i^{(0)}
\right)
\right].
\end{equation}

This quantity measures how strongly the intervention changes the semantic context delivered by graph relations while leaving the node content itself unchanged.

For each intervened graph, we keep the GNN architecture fixed and optimize only Encoder-LoRA parameters together with the task head. Rather than comparing raw LoRA factors, whose factorization is non-unique, we analyze the effective update of each adapted layer,

\begin{equation}
\Delta\mathbf{W}_{\ell}^{(\tau)}
=
\mathbf{B}_{\ell}^{(\tau)}
\mathbf{A}_{\ell}^{(\tau)}.
\end{equation}

We concatenate the normalized effective updates across adapted Encoder layers,

\begin{equation}
\mathbf{u}_{E}^{(\tau)}
=
\operatorname{vec}
\left(
\left\{
\frac{
\Delta\mathbf{W}_{\ell}^{(\tau)}
}{
\|\Delta\mathbf{W}_{\ell}^{(\tau)}\|_{F}+\epsilon
}
\right\}_{\ell\in\mathcal{L}_{E}}
\right),
\end{equation}

and measure its deviation from the original relational environment as

\begin{equation}
d_{\mathrm{upd}}^{(\tau)}
=
1-
\cos
\left(
\mathbf{u}_{E}^{(\tau)},
\mathbf{u}_{E}^{(0)}
\right).
\end{equation}

To compare the relational intervention with a lower-semantic-shift control,
we construct perturbations that produce smaller observed changes in
neighborhood semantic profiles. We refer to these as low-semantic-shift
perturbation controls rather than semantic-preserving placebo interventions,
because neighborhood label homophily is not explicitly constrained. The diagnostic therefore compares the observed Encoder-update displacement
under the relational interventions and the low-semantic-shift controls.

\begin{equation}
\mathcal{R}
=
\frac{
d_{\mathrm{upd}}^{\mathrm{rel}}
}{
d_{\mathrm{upd}}^{\mathrm{plc}}+\epsilon
},
\end{equation}

together with the dose--response coefficient between relational shift and Encoder-update movement,

\begin{equation}
\widehat{\beta}
=
\frac{
\sum_{\tau}
\left(D_{\mathrm{rel}}^{(\tau)}-\overline{D}_{\mathrm{rel}}\right)
\left(d_{\mathrm{upd}}^{(\tau)}-\overline{d}_{\mathrm{upd}}\right)
}{
\sum_{\tau}
\left(D_{\mathrm{rel}}^{(\tau)}-\overline{D}_{\mathrm{rel}}\right)^2+\epsilon
}.
\end{equation}

As reported in Sec.~3, the strongest relational intervention produces an
update distance of $0.088$, compared with $0.036$ under the low-semantic-shift
control, yielding a $2.4\times$ difference. The corresponding dose--response
slopes are $0.70$ and $0.24$. These results do not imply that topology
heterogeneity is necessarily stronger than feature heterogeneity. They provide
controlled evidence that graph-relational perturbations are associated with
systematic changes in the preferred Encoder adaptation. Because neighborhood
label homophily is not explicitly controlled, the result does not establish a
pure topology-only causal effect. This motivates representing perception
adaptation as a relation-sensitive state rather than as an isolated
multimodal adapter.

\subsubsection{Cross-Client Encoder--GNN Compatibility}

The preceding study shows that the Encoder update depends on the local relational environment. We therefore examine whether Encoder and GNN adaptations learned within the same client remain independently interchangeable. For each client $k$, we obtain a locally adapted pair $(E_k,G_k)$ and evaluate every cross-client composition $(E_i,G_j)$. This produces the compatibility matrix

\begin{equation}
\mathbf{M}_{ij}
=
\mathcal{S}
\left(
E_i,G_j;\mathcal{D}
\right),
\quad
i,j\in\{1,\ldots,K\},
\end{equation}

where $\mathcal{S}(\cdot)$ denotes downstream performance under a common evaluation protocol. Diagonal entries correspond to Encoder--GNN adaptations formed under the same client environment, whereas off-diagonal entries deliberately break this locally learned correspondence.

We quantify this effect through the pairing gap

\begin{equation}
\Delta_{\mathrm{pair}}
=
\frac{1}{K}
\sum_{i=1}^{K}\mathbf{M}_{ii}
-
\frac{1}{K(K-1)}
\sum_{i\neq j}\mathbf{M}_{ij}.
\end{equation}

A large $\Delta_{\mathrm{pair}}$ indicates that the Encoder and GNN do not behave as independently exchangeable adapters: their effectiveness depends on being paired with the complementary state learned under the same relational environment. We apply the same evaluation protocol to the three parameterizations. Because the baselines and FedCORE also differ in their adapter coordinates, this diagnostic compares end-to-end compatibility and does not by itself isolate the effect of core sharing from coordinate anchoring.

The observed pairing gaps are $30.6$ for Independent Joint LoRA, $22.8$ after module-wise federated aggregation, and $5.9$ for FedCORE. Independent Joint LoRA therefore demonstrates that simply optimizing both modules with the same local objective is insufficient to make their adaptations independently composable. Module-wise aggregation partially reduces the discrepancy but still aggregates two marginal parameter states rather than the state that generates their compatibility. FedCORE instead constrains both module updates to be realizations of the same client-specific core, substantially reducing the mismatch while preserving strong matched-pair performance.

These three diagnostics form the empirical basis of FedCORE. The capacity study establishes the need to adapt multimodal perception, the intervention study shows that such adaptation is conditioned by graph relations, and the pairing study shows that perception and graph reasoning form a compatible client-specific state. The resulting design principle is therefore not simply to make both modules trainable, but to \emph{jointly represent} their relation-conditioned adaptation through a common state that can survive federated aggregation.

\subsection{Theoretical Properties of Shared-Core Adaptation}
\label{app:theory}

\setcounter{proposition}{0}
\renewcommand{\theproposition}{A.\arabic{proposition}}

We provide several theoretical properties that connect the empirical
observations in Sec.~\ref{sec:empirical} with the design of FedCORE. Rather
than establishing a generic convergence guarantee for federated optimization,
our analysis focuses on four questions specific to our setting: why graph
relations can affect Encoder adaptation, how the shared core constrains
Encoder--GNN compatibility, why such compatibility is preserved after
federated aggregation, and under what conditions a local geometry-balanced
update decreases the objective. These propositions characterize the shared
linear parameterization and a single preconditioned local step. They do not
constitute a convergence guarantee for multi-step federated optimization with
client drift or state-quality weighting.

\paragraph{Relation-conditioned Encoder adaptation.}

Consider the local prediction function
\begin{equation}
\widehat{\mathbf{Y}}_k
=
H_{\omega_k}
\left(
G_{\Phi(\mathbf{C}_k)}
\left(
E_{\Theta(\mathbf{C}_k)}(\mathbf{X}_k),
\mathbf{A}_k
\right)
\right),
\end{equation}
where both Encoder and GNN parameters are induced by the client core
$\mathbf{C}_k$. Let $\ell_k$ denote the task loss.

\begin{proposition}[Relation-conditioned Encoder gradient]
\label{prop:relation_gradient}
Let
\[
\mathbf{H}_k
=
E_{\Theta(\mathbf{C}_k)}(\mathbf{X}_k)
\]
and assume that the GNN, task head, and task loss are differentiable. The
gradient received by the Encoder satisfies
\begin{equation}
\nabla_{\mathbf{H}_k}\ell_k
=
\mathbf{J}_{G}^{\top}
(\mathbf{H}_k,\mathbf{A}_k)
\mathbf{J}_{H}^{\top}
\nabla_{\widehat{\mathbf{Y}}_k}\ell_k,
\end{equation}
where $\mathbf{J}_{G}(\mathbf{H}_k,\mathbf{A}_k)$ is the Jacobian of graph
propagation with respect to its input representations and $\mathbf{J}_{H}$
is the Jacobian of the task head. Therefore, even when $\mathbf{X}_k$ and
the labels remain unchanged, modifying $\mathbf{A}_k$ can alter the gradient
applied to the multimodal Encoder whenever
\begin{equation}
\mathbf{J}_{G}(\mathbf{H}_k,\mathbf{A}_k)
\neq
\mathbf{J}_{G}(\mathbf{H}_k,\mathbf{A}'_k).
\end{equation}
\end{proposition}

\begin{proof}
By the chain rule,
\begin{equation}
\nabla_{\mathbf{H}_k}\ell_k
=
\left(
\frac{\partial
G_{\Phi(\mathbf{C}_k)}(\mathbf{H}_k,\mathbf{A}_k)}
{\partial \mathbf{H}_k}
\right)^{\top}
\left(
\frac{\partial H_{\omega_k}}
{\partial G_{\Phi(\mathbf{C}_k)}}
\right)^{\top}
\nabla_{\widehat{\mathbf{Y}}_k}\ell_k,
\end{equation}
which gives the stated expression. Since message passing explicitly depends
on $\mathbf{A}_k$, its Jacobian generally changes when the relational
structure changes.

For example, consider a message-passing layer
\begin{equation}
\mathbf{Z}
=
\sigma
\left(
\widetilde{\mathbf{A}}_k
\mathbf{H}_k
\mathbf{W}_{G}
\right).
\end{equation}
Let
\[
\mathbf{Q}
=
\nabla_{\mathbf{Z}}\ell_k
\odot
\sigma'
\left(
\widetilde{\mathbf{A}}_k
\mathbf{H}_k
\mathbf{W}_{G}
\right).
\]
Then
\begin{equation}
\nabla_{\mathbf{H}_k}\ell_k
=
\widetilde{\mathbf{A}}_k^{\top}
\mathbf{Q}
\mathbf{W}_{G}^{\top},
\end{equation}
which explicitly depends on the graph relations. Hence, changing only
$\mathbf{A}_k$ is sufficient to change the learning signal received by the
Encoder, except in degenerate cases where the relevant Jacobians or
downstream gradients coincide.
\end{proof}

Proposition~\ref{prop:relation_gradient} provides a mechanistic explanation
for Observation~2. Graph relations do not merely determine how fixed
representations are propagated; through back-propagation, they also condition
how the upstream multimodal representation should be adapted.

\paragraph{A shared compatibility manifold.}

For an adapted module $j$, FedCORE decodes the shared core as
\begin{equation}
\mathcal{D}_j(\mathbf{C})
=
\alpha_j
\mathbf{U}_j
\mathbf{S}_j
\mathbf{C}
\mathbf{S}_j
\mathbf{V}_j^{\top},
\end{equation}
where $\mathbf{U}_j$ and $\mathbf{V}_j$ are fixed rank-$r$ coordinate bases
satisfying
\[
\mathbf{U}_j^{\top}\mathbf{U}_j
=
\mathbf{V}_j^{\top}\mathbf{V}_j
=
\mathbf{I}_r,
\]
$\mathbf{S}_j$ is a fixed positive diagonal scaling matrix, and
$\alpha_j>0$ is a fixed module-specific scale. Its canonical projection is
\begin{equation}
\mathcal{P}_j(\Delta\mathbf{W})
=
\alpha_j^{-1}
\mathbf{S}_j^{-1}
\mathbf{U}_j^{\top}
\Delta\mathbf{W}
\mathbf{V}_j
\mathbf{S}_j^{-1}.
\end{equation}

\begin{proposition}[Cross-module compatibility]
\label{prop:compatibility}
For every adapted module $j$,
\begin{equation}
\mathcal{P}_j
\left(
\mathcal{D}_j(\mathbf{C})
\right)
=
\mathbf{C}.
\end{equation}
Consequently, all module updates generated by FedCORE belong to the coupled
adaptation space
\begin{equation}
\mathcal{M}
=
\left\{
\left(
\mathcal{D}_1(\mathbf{C}),
\ldots,
\mathcal{D}_{J}(\mathbf{C})
\right)
:
\mathbf{C}\in\mathbb{R}^{r\times r}
\right\},
\end{equation}
which has dimension $r^2$.
\end{proposition}

\begin{proof}
Substituting $\mathcal{D}_j(\mathbf{C})$ into $\mathcal{P}_j$ gives
\begin{align}
\mathcal{P}_j(\mathcal{D}_j(\mathbf{C}))
&=
\mathbf{S}_j^{-1}
\mathbf{U}_j^{\top}
\mathbf{U}_j
\mathbf{S}_j
\mathbf{C}
\mathbf{S}_j
\mathbf{V}_j^{\top}
\mathbf{V}_j
\mathbf{S}_j^{-1} \\
&=
\mathbf{C}.
\end{align}
Hence $\mathcal{D}_j$ is injective for every $j$. The joint decoding
\begin{equation}
\mathcal{D}(\mathbf{C})
=
\left(
\mathcal{D}_1(\mathbf{C}),
\ldots,
\mathcal{D}_{J}(\mathbf{C})
\right)
\end{equation}
is therefore also injective. Since its domain is
$\mathbb{R}^{r\times r}$ and all $\mathcal{D}_j$ are linear, $\mathcal{M}$ is
a linear subspace of dimension $r^2$.
\end{proof}

This property gives a precise interpretation of a \emph{shared adaptation
state}: although different modules receive heterogeneous parameter updates,
every update can be uniquely projected back to the same core.

To quantify deviations from this condition, for arbitrary independently
learned module updates $\{\Delta\mathbf{W}_j\}_{j=1}^{J}$, define the
compatibility defect
\begin{equation}
\mathcal{I}
\left(
\{\Delta\mathbf{W}_j\}_{j=1}^{J}
\right)
=
\min_{\mathbf{C}}
\sum_{j=1}^{J}
\omega_j
\left\|
\mathcal{P}_j(\Delta\mathbf{W}_j)
-
\mathbf{C}
\right\|_F^2,
\qquad
\omega_j>0.
\end{equation}

Its minimizer has the closed form
\begin{equation}
\mathbf{C}^{\star}
=
\frac{
\sum_{j=1}^{J}
\omega_j
\mathcal{P}_j(\Delta\mathbf{W}_j)
}{
\sum_{j=1}^{J}\omega_j
}.
\end{equation}

For FedCORE,
\begin{equation}
\mathcal{I}
\left(
\{\mathcal{D}_j(\mathbf{C})\}_{j=1}^{J}
\right)
=
0,
\end{equation}
whereas independently parameterized Encoder and GNN adapters generally have
a positive defect whenever their projected latent states disagree. This
provides a parameter-space interpretation of the cross-module incompatibility
observed in Observation~3.

\paragraph{Aggregation closure of the coupled state.}

\begin{proposition}[Aggregation consistency]
\label{prop:aggregation}
Let $\mathbf{C}_k$ denote the locally evolved core of client $k$ and let
\begin{equation}
\overline{\mathbf{C}}
=
\sum_{k=1}^{K}p_k\mathbf{C}_k,
\qquad
\sum_{k=1}^{K}p_k=1.
\end{equation}
Then, for every module $j$,
\begin{equation}
\mathcal{D}_j(\overline{\mathbf{C}})
=
\sum_{k=1}^{K}
p_k
\mathcal{D}_j(\mathbf{C}_k),
\end{equation}
and the resulting global module updates remain in $\mathcal{M}$ with zero
compatibility defect.
\end{proposition}

\begin{proof}
Since $\mathcal{D}_j$ is linear,
\begin{align}
\mathcal{D}_j
\left(
\sum_{k=1}^{K}p_k\mathbf{C}_k
\right)
&=
\alpha_j\mathbf{U}_j\mathbf{S}_j
\left(
\sum_{k=1}^{K}p_k\mathbf{C}_k
\right)
\mathbf{S}_j\mathbf{V}_j^{\top} \\
&=
\sum_{k=1}^{K}
p_k\mathcal{D}_j(\mathbf{C}_k).
\end{align}
Moreover,
\begin{equation}
\mathcal{P}_j
\left(
\mathcal{D}_j(\overline{\mathbf{C}})
\right)
=
\overline{\mathbf{C}}
\end{equation}
for every $j$ by Proposition~\ref{prop:compatibility}. Therefore all
aggregated module updates recover the same global state and remain in
$\mathcal{M}$.
\end{proof}

Under the same linear decoder and the same aggregation weights, aggregating
the shared cores before decoding is equivalent to decoding each local core
first and then averaging the resulting module updates. Therefore, a
module-wise aggregation variant is different from FedCORE only when it changes
the local parameterization, uses module-specific aggregation weights, or
introduces additional module-specific operations.

Proposition~\ref{prop:aggregation} shows that FedCORE restricts local updates
to the shared compatibility family $\mathcal{M}$, which is closed under linear
federated averaging. Independently learned Encoder and GNN updates need not
share a common projected state, leading to cross-client combination
incompatibility. Under fixed linear decoders and common aggregation weights,
core-space aggregation and decoded module-wise averaging are algebraically
equivalent; FedCORE's structural benefit is the shared parameterization and
explicit compatibility constraint.

\paragraph{Exact joint gradient pullback.}

The shared core is not optimized by heuristically averaging Encoder- and
GNN-side gradients. Its gradient follows exactly from the joint objective.
Assume that the core regularization term in the local objective is
\[
\frac{\lambda_r}{2}\|\mathbf{C}_k\|_F^2.
\]

\begin{proposition}[Exact cross-module pullback]
\label{prop:pullback}
Let
\begin{equation}
\mathbf{G}_{k,j}^{W}
=
\nabla_{\mathbf{W}_{k,j}}\mathcal{L}_k.
\end{equation}
Under the Frobenius inner product, the adjoint of $\mathcal{D}_j$ is
\begin{equation}
\mathcal{D}_j^{\ast}(\mathbf{G})
=
\alpha_j
\mathbf{S}_j
\mathbf{U}_j^{\top}
\mathbf{G}
\mathbf{V}_j
\mathbf{S}_j.
\end{equation}
Therefore,
\begin{equation}
\nabla_{\mathbf{C}_k}\mathcal{L}_k
=
\sum_{j\in\mathcal{J}}
\mathcal{D}_j^{\ast}
\left(
\mathbf{G}_{k,j}^{W}
\right)
+
\lambda_r\mathbf{C}_k.
\end{equation}
\end{proposition}

\begin{proof}
For an arbitrary perturbation $\mathbf{Z}$,
\begin{align}
\left\langle
\mathbf{G},
\mathcal{D}_j(\mathbf{Z})
\right\rangle_F
&=
\alpha_j
\operatorname{tr}
\left(
\mathbf{G}^{\top}
\mathbf{U}_j\mathbf{S}_j
\mathbf{Z}\mathbf{S}_j\mathbf{V}_j^{\top}
\right) \\
&=
\left\langle
\alpha_j\mathbf{S}_j
\mathbf{U}_j^{\top}
\mathbf{G}
\mathbf{V}_j\mathbf{S}_j,
\mathbf{Z}
\right\rangle_F.
\end{align}
Hence the stated expression is the adjoint mapping. Applying the chain rule
over all decoded modules and differentiating
$\lambda_r\|\mathbf{C}_k\|_F^2/2$ gives the result.
\end{proof}

This shows that Encoder- and GNN-side learning signals are exactly transported
into the same coordinate system and jointly determine the evolution of
$\mathbf{C}_k$. In particular, by Proposition~\ref{prop:relation_gradient},
the Encoder-side component already contains information about the local graph
relations through back-propagation.

\paragraph{Geometry-balanced core optimization.}

The second-moment statistics and preconditioned update below are the same
quantities used in the local optimization procedure in Sec.~5.3. The
sufficient descent result concerns one local update with fixed preconditioners;
it does not establish convergence of the complete federated procedure.
Let
\[
\mathbf{G}_k^e
=
\nabla_{\mathbf{C}_k^e}\mathcal{L}_k^e
\]
denote the core gradient at local step $e$. The client maintains the left and
right second-moment statistics
\begin{align}
\mathbf{L}_k^e
&=
\beta\mathbf{L}_k^{e-1}
+
(1-\beta)
\mathbf{G}_k^e(\mathbf{G}_k^e)^{\top},\\
\mathbf{R}_k^e
&=
\beta\mathbf{R}_k^{e-1}
+
(1-\beta)
(\mathbf{G}_k^e)^{\top}\mathbf{G}_k^e,
\end{align}
where $\beta\in[0,1)$ is a decay coefficient. These statistics are maintained
locally and are not communicated to the server. The resulting update follows the general family of matrix preconditioned
optimization methods that maintain left and right gradient statistics
\citep{gupta2018shampoo}.

For readability, we omit the local-step superscript below and define
\begin{equation}
\mathbf{P}_{L}
=
(\mathbf{L}_k+\epsilon\mathbf{I})^{-1/4},
\qquad
\mathbf{P}_{R}
=
(\mathbf{R}_k+\epsilon\mathbf{I})^{-1/4},
\end{equation}
where $\epsilon>0$ ensures positive definiteness. The geometry-balanced core
update is
\begin{equation}
\mathbf{C}^{+}
=
\mathbf{C}
-
\eta
\mathbf{P}_{L}
\nabla_{\mathbf{C}}\mathcal{L}
\mathbf{P}_{R}.
\end{equation}

\begin{proposition}[Sufficient descent under fixed preconditioning]
\label{prop:descent}
Assume $\mathcal{L}(\mathbf{C})$ is $L_C$-smooth and
$\mathbf{P}_{L},\mathbf{P}_{R}$ are symmetric positive definite and fixed
during one update. Let
\begin{equation}
\mathbf{M}
=
\mathbf{P}_{R}\otimes\mathbf{P}_{L}.
\end{equation}
If
\begin{equation}
0<\eta<
\frac{
2\lambda_{\min}(\mathbf{M})
}{
L_C\lambda_{\max}^2(\mathbf{M})
},
\end{equation}
then every non-stationary update yields
\begin{equation}
\mathcal{L}(\mathbf{C}^{+})
<
\mathcal{L}(\mathbf{C}).
\end{equation}
\end{proposition}

\begin{proof}
Let
\[
\mathbf{g}
=
\operatorname{vec}
\left(
\nabla_{\mathbf{C}}\mathcal{L}
\right).
\]
Vectorizing the update gives
\begin{equation}
\operatorname{vec}(\mathbf{C}^{+}-\mathbf{C})
=
-\eta\mathbf{M}\mathbf{g}.
\end{equation}

By $L_C$-smoothness,
\begin{align}
\mathcal{L}(\mathbf{C}^{+})
&\leq
\mathcal{L}(\mathbf{C})
-
\eta\mathbf{g}^{\top}\mathbf{M}\mathbf{g}
+
\frac{L_C\eta^2}{2}
\|\mathbf{M}\mathbf{g}\|_2^2 \\
&\leq
\mathcal{L}(\mathbf{C})
-
\eta
\left[
\lambda_{\min}(\mathbf{M})
-
\frac{L_C\eta}{2}
\lambda_{\max}^2(\mathbf{M})
\right]
\|\mathbf{g}\|_2^2.
\end{align}

The stated step-size condition makes the term in brackets strictly positive,
yielding strict descent whenever $\mathbf{g}\neq 0$.
\end{proof}

The analysis therefore establishes a direct connection between our empirical
findings and method design. Graph relations can systematically alter the
Encoder learning signal; Encoder and GNN updates generated from the same core
lie on a common $r^2$-dimensional compatibility manifold; this manifold is
preserved by linear federated aggregation; and a single local update with a
fixed positive-definite preconditioner satisfies a sufficient descent
condition.

These results explain the structural motivation of FedCORE, but they do not
provide a convergence guarantee for the complete federated optimization
process. In particular, they do not analyze client drift over multiple local
steps, partial client participation, state-quality weighting, or the
interaction between local preconditioner statistics and server-side state
evolution.

\subsection{Experimental Setup}
\label{app:setup}

\subsubsection{Datasets and Tasks}

We evaluate FedCORE on multimodal graph benchmarks covering diverse
downstream tasks, including node classification, link prediction,
multimodal retrieval, graph-to-text generation, and graph-to-image
generation. All datasets follow the federated multimodal graph learning protocol introduced by MM-OpenFGL~\citep{li2026mm}, where each client owns a local multimodal graph and raw data remain decentralized.

\paragraph{Amazon product graphs.}
Toys and Grocery are multimodal Amazon product graphs. Each node represents
an individual product with textual and visual attributes, while product--product
relations, such as co-purchase or complementary-product relations, define the
graph edges. Toys contains 20,695 nodes and 18 product categories, whereas
Grocery contains 17,074 nodes and 20 product categories. We evaluate node
classification on both datasets and report classification accuracy.

\paragraph{NineRec-based datasets.}
Bili Music, DY, QB, and Bili Cartoon are derived from the NineRec benchmark
\citep{zhang2024ninerec}. These datasets contain multimodal user--item
interaction graphs with heterogeneous semantic information. Following
MM-OpenFGL~\citep{li2026mm}, we evaluate link prediction on Bili Music
and DY using AUC. For QB and Bili Cartoon, we evaluate multimodal retrieval
using Recall@5, measuring whether the relevant item appears among the
top-five retrieved candidates.

\paragraph{Flickr30k graph-to-text generation.}
For graph-to-text generation, we use the preprocessed Flickr30k graph with
31,783 nodes and partition it into five clients using
\texttt{subgraph\_fl\_louvain}. Visual features are extracted using CLIP
ViT-L/14 with dimension 768, while textual features are extracted using
RoBERTa-base. We use the trainable
\texttt{flan-t5-base-VG-factual-sg} decoder with
\texttt{freeze\_lm=false} and no PEFT adapter. Neighbor information is
constructed using the embedding-based neighbor mode, with at most four text
neighbors, four image neighbors, and one neighbor layer. The maximum input
and output lengths are 1024 and 128 tokens, respectively. The model is
trained for five epochs with a learning rate of $1\times10^{-3}$, a batch size
of 4, weight decay of $1\times10^{-2}$, and gradient accumulation over four
steps. Generation quality is evaluated using ROUGE-L.

\paragraph{SemArt graph-to-image generation.}
For graph-to-image generation, we use the preprocessed SemArt graph with
21,382 nodes and partition it into five clients using
\texttt{subgraph\_fl\_louvain}. The graph-conditioned generator is based on
Stable Diffusion v1.5 with a resolution of $256\times256$. Graph information
is incorporated through a GNN, a GraphAdapter, and a Q-Former. We use four
graph neighbors, a cross-attention frequency of 2, and four additional CLIP
context tokens. The model is trained for 10 epochs with a learning rate of
$1\times10^{-4}$, a batch size of 16, and Adam weight decay of
$1\times10^{-2}$. Generation quality is evaluated using CLIP-Score.

The datasets and their downstream tasks are described above, and all
experiments follow the federated multimodal graph learning protocol of
MM-OpenFGL.

\subsubsection{Evaluation Metrics}

For node classification tasks, we report Accuracy, defined as the proportion of correctly classified nodes:

\begin{equation}
\mathrm{Accuracy}
=
\frac{1}{N}
\sum_{i=1}^{N}
\mathbb{I}
(\hat y_i=y_i).
\end{equation}

For link prediction tasks, we use Area Under the ROC Curve (AUC), which measures the probability that a randomly sampled positive edge receives a higher score than a negative edge.

For multimodal retrieval tasks, we adopt Recall@5 (R@5), which evaluates the percentage of queries whose corresponding targets appear in the top-5 retrieved candidates.

For graph-to-text generation, we report ROUGE-L, which measures the longest-common-subsequence overlap between the generated text and its
reference text.

For graph-to-image generation, we report CLIP-Score, computed as the cosine
similarity between the CLIP embeddings of the generated image and its corresponding target image. We use the pretrained CLIP ViT-L/14 model with
its standard image and text preprocessing.

\subsubsection{Baselines}

We compare FedCORE with representative baselines from three categories.

\textbf{Federated Learning Baseline.}
FedAvg~\citep{mcmahan2017communication} is the standard federated optimization method that independently averages client model parameters.

\textbf{Federated Graph Foundation Model Baselines.}
FedGFM+~\citep{zhu2025fedgfm}, STAGE~\citep{chen2026stage}, and FedGAMMA~\citep{li2026fedgamma} are representative approaches for adapting pretrained graph foundation models under federated settings. These methods investigate parameter-efficient adaptation, knowledge sharing, and cross-client transferability.

\textbf{Federated Multimodal Graph Baseline.}
Federated PLANET~\citep{liu2026toward} extends multimodal graph foundation models to decentralized environments and serves as a strong comparison for multimodal graph adaptation.

\paragraph{Adaptation controls.}
To analyze the contribution of Encoder--GNN coordination, we include two
adaptation controls.

\emph{Frozen Encoder + GNN} updates only graph-side parameters while keeping
the pretrained multimodal Encoder fixed. This baseline evaluates whether
graph reasoning alone can achieve effective adaptation.

\emph{Independent Joint LoRA + FedAvg} is implemented as an isolated
matched-backbone control. The common backbone uses the frozen visual,
textual, and fusion projections from FedCOREv2, with input features extracted
by CLIP ViT-L/14. Independent LoRA adapters are applied to the visual
projection, textual projection, fusion projection, and the three custom
residual GNN transformation matrices. The GNN contains three message-passing
layers with hidden dimension 256. For each adapted matrix, the update is
parameterized as
\[
\Delta \mathbf{W}_j
=
\frac{\alpha}{r}\mathbf{B}_j\mathbf{A}_j,
\]
where $\alpha=4r$ and no LoRA dropout is used. We evaluate ranks
$r\in\{1,2,4,8,16,32\}$ in the rank-sensitivity experiments; the
adaptation-scale comparison uses $r=24$.

The LoRA factors are optimized independently on each client and aggregated
using FedAvg. The same five clients, full client participation, warm-up
schedule, local epochs, optimizer, and data splits are used as for FedCORE.
The local task classifier is trainable but remains client-specific: parameters
whose names start with \texttt{classifier.} are excluded from the federated
payload and are not averaged by the server. Thus, the reported trainable
parameter count includes both LoRA factors and the local task head, whereas
the communication payload counts only the LoRA $\mathbf{A}_j$ and
$\mathbf{B}_j$ tensors.

\subsubsection{Implementation Details}

All methods are evaluated under identical data splits, pretrained backbones,
and downstream settings. Following MM-OpenFGL~\citep{li2026mm}, each
dataset is partitioned into five clients unless otherwise specified. For node
classification, multimodal retrieval, graph-to-text generation, and
graph-to-image generation, we use Louvain community detection to construct
the client subgraphs. For link prediction, we use Metis graph partitioning\citep{wang2022federatedscope}.
All compared methods use the same client partitions.

For FedCORE, only the upper layers of the multimodal Encoder, fusion layers,
and primary GNN transformations are adapted, while lower backbone layers
remain frozen. The latent state dimension and method-specific optimization
hyperparameters are selected using the validation set and fixed before
evaluation. For the Independent Joint LoRA control, clients use two local epochs with
Adam optimization, learning rate $0.004$, weight decay $0.0012$, batch size
512, and dropout rate $0.38$. The warm-up stage lasts 10 rounds, followed by
45 adaptation rounds. All five clients participate in every communication
round.

\section{Additional Efficiency and Scalability Analysis}
\label{app:efficiency_scalability}

\textbf{Convergence Efficiency.}
Figure~\ref{fig:efficiency_scalability}(a) compares the performance
of FedCORE and representative baselines across communication rounds
on Toys node classification. FedCORE reaches a strong performance
level within fewer rounds and remains stable afterwards, providing
additional evidence that preserving Encoder--GNN coordination can
improve federated optimization.

\textbf{Scalability to Adaptation Scale.}
Figure~\ref{fig:efficiency_scalability}(b) examines how the methods
behave as the adaptation scale increases while the core dimension of
FedCORE remains fixed. Independent Joint LoRA communicates more
module-specific parameters as additional layers are adapted, whereas
FedCORE continues to transmit a compact $r\times r$ core. The results
provide additional evidence of the performance--communication trade-off
associated with increasing the adaptation scale.

\begin{figure*}[htbp]
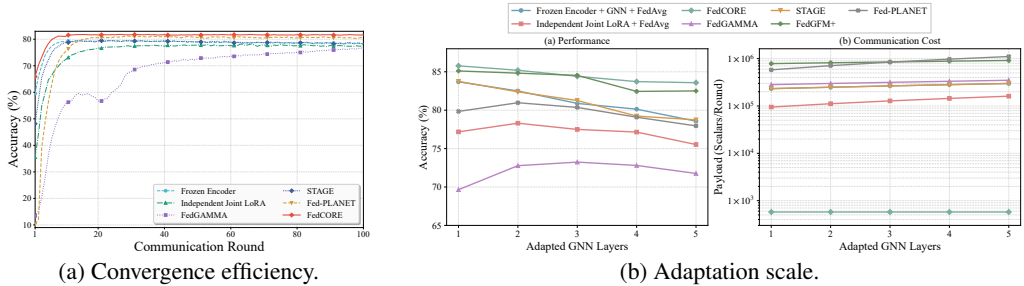

\centering

\begin{minipage}[t]{0.35\textwidth}
\centering
\includegraphics[width=\linewidth]
{figure/convergence\_toys\_nodecls\_baselines\_allbands\_no\_shadow.pdf}\\[-2pt]
\small (a) Convergence efficiency.
\end{minipage}
\hspace{0.025\textwidth}
\begin{minipage}[t]{0.58\textwidth}
\centering
\includegraphics[width=\linewidth]
{figure/grocery\_nodecls\_adaptation\_scale\_tradeoff\_with\_baselines.pdf}\\[-2pt]
\small (b) Adaptation scale.
\end{minipage}

\vspace{-4pt}

\caption{
\textbf{Additional convergence and scalability analysis.}
(a) Test accuracy on Toys node classification across communication
rounds.
(b) Performance--communication trade-off on Grocery node classification
as the adaptation scale increases.
}
\label{fig:efficiency_scalability}
\end{figure*}
\section{Additional Robustness Analysis}
\label{app:additional_robustness}

\textbf{Robustness to Client Heterogeneity.}
We evaluate FedCORE on Toys node classification under client-specific feature
drift and topology perturbation. We keep the node set, labels, data splits,
and training settings fixed. For feature drift, we add a shared offset
$b_k\sim\mathcal{N}(0,\alpha^2I)$ to all nodes of client $k$, with
$\alpha\in\{0,0.15,0.30,0.45,0.60\}$. For topology perturbation, we inject
undirected edges between nodes from different Louvain communities before the
final five-client partition, using
$\rho\in\{0,0.0625,0.125,0.1875,0.25\}$. As shown in
Figure~\ref{fig:heterogeneity}, FedCORE degrades more slowly than the
baselines as either perturbation increases.

\begin{figure*}[t]
    \centering
    \includegraphics[width=0.8\linewidth]
    {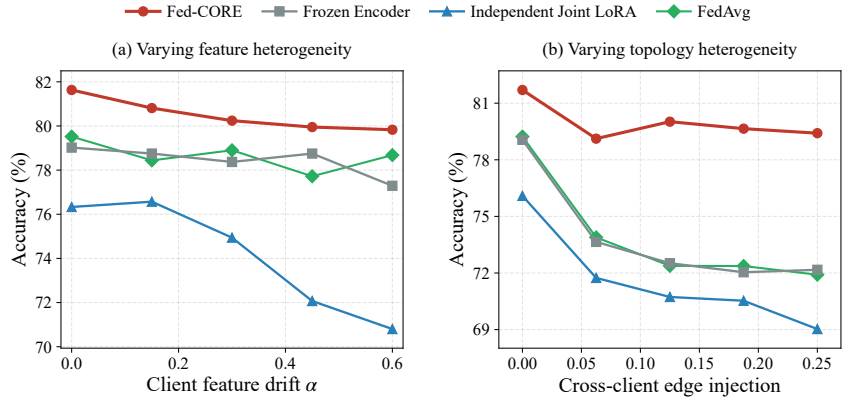}
    \caption{
    \textbf{Robustness to client heterogeneity.}
    The left panel varies client feature drift $\alpha$, while the right panel
    varies the cross-community edge-injection ratio $\rho$ before the final
    client partition. The node set, labels, and data splits remain fixed.
    }
    \label{fig:heterogeneity}
\end{figure*}

\appendix

\end{document}